\documentclass{article}

\usepackage{microtype}
\usepackage{graphicx}
\usepackage{subfigure}
\usepackage{booktabs} 

\usepackage{hyperref}

\usepackage{amsmath,amsfonts,bm, amsthm, amssymb}

\def\1{\bm{1}}

\DeclareMathAlphabet{\mathsfit}{\encodingdefault}{\sfdefault}{m}{sl}
\SetMathAlphabet{\mathsfit}{bold}{\encodingdefault}{\sfdefault}{bx}{n}

\def\mcA{{\mathcal{A}}}
\def\mcB{{\mathcal{B}}}

\def\mcD{{\mathcal{D}}}

\def\mcG{{\mathcal{G}}}

\def\mcM{{\mathcal{M}}}
\def\mcN{{\mathcal{N}}}

\def\mcP{{\mathcal{P}}}

\def\mcR{{\mathcal{R}}}
\def\mcS{{\mathcal{S}}}

\def\mcX{{\mathcal{X}}}

\newcommand{\levy}{L\'evy-Prokhorov\ }
\newcommand{\lp}[1]{\pi(#1)}
\newcommand{\ours}[2]{\Pi_{#1}(#2)}
\newcommand{\oursapprox}[2]{\overline{\Pi}_{#1}(#2)}
\newcommand{\mass}[1]{\mcM(#1)}

\newcommand{\excess}{\mathrm{Exc}}
\newcommand{\ignore}[1]{}
\newcommand{\cell}{\Box}

\usepackage[accepted]{icml2024}

\usepackage{amsmath}
\usepackage{amssymb}
\usepackage{mathtools}
\usepackage{amsthm, thmtools}

\usepackage[capitalize,noabbrev]{cleveref}

\theoremstyle{plain}
\newtheorem{theorem}{Theorem}[section]

\newtheorem{lemma}[theorem]{Lemma}

\theoremstyle{definition}

\theoremstyle{remark}

\usepackage[textsize=tiny]{todonotes}

\icmltitlerunning{A New Robust Partial \texorpdfstring{$p$}{}-Wasserstein-Based Metric for Comparing Distributions}

\begin{document}

\twocolumn[
\icmltitle{A New Robust Partial \texorpdfstring{$p$}{}-Wasserstein-Based Metric for Comparing Distributions}



\icmlsetsymbol{equal}{*}

\begin{icmlauthorlist}
\icmlauthor{Sharath Raghvendra}{ncsu,equal}
\icmlauthor{Pouyan Shirzadian}{vt}
\icmlauthor{Kaiyi Zhang}{vt}
\end{icmlauthorlist}

\icmlaffiliation{ncsu}{North Carolina State University}
\icmlaffiliation{vt}{Virginia Tech}


\icmlkeywords{Machine Learning, ICML}

\vskip 0.3in
]



\printAffiliationsAndNotice{\icmlEqualContribution} 

\begin{abstract}
The $2$-Wasserstein distance is sensitive to minor geometric differences between distributions, making it a very powerful dissimilarity metric. However, due to this sensitivity, a small outlier mass can also cause a significant increase in the $2$-Wasserstein distance between two similar distributions. Similarly, sampling discrepancy can cause the empirical $2$-Wasserstein distance on $n$ samples in $\mathbb{R}^2$ to converge to the true distance at a rate of $n^{-1/4}$, which is significantly slower than the rate of $n^{-1/2}$ for $1$-Wasserstein distance.
We introduce a new family of distances parameterized by $k \ge 0$, called $k$-RPW that is based on computing the partial $2$-Wasserstein distance. We show that (1) $k$-RPW satisfies the metric properties, (2) $k$-RPW is robust to small outlier mass while retaining the sensitivity of $2$-Wasserstein distance to minor geometric differences, and (3) when $k$ is a constant, $k$-RPW distance between empirical distributions on $n$ samples in $\mathbb{R}^2$ converges to the true distance at a rate of $n^{-1/3}$, which is faster than the convergence rate of $n^{-1/4}$ for the $2$-Wasserstein distance.
Using the partial $p$-Wasserstein distance, we extend our distance to any $p \in [1,\infty]$.
By setting parameters $k$ or $p$ appropriately, we can reduce our distance to the total variation, $p$-Wasserstein, and the L\'evy-Prokhorov distances. Experiments show that our distance function achieves higher accuracy in comparison to the $1$-Wasserstein, $2$-Wasserstein, and TV distances for image retrieval tasks on noisy real-world data sets.

\end{abstract}

\section{Introduction}\label{sec:introduction}
Given two probability distributions $\mu$ and $\nu$ with supports $\mcA$ and $\mcB$, let, for any $(a,b) \in \mcA\times \mcB$, $d(a,b)$ be the cost of moving a unit mass from $a$ to $b$. A \emph{transport plan} $\gamma$ is a coupling of $\mu$ and $\nu$, i.e., a joint distribution over the support $\mcA\times \mcB$ whose first and second marginals are $\mu$ and $\nu$. For $p \ge 1$, consider the case where the support of $\mu$ and $\nu$ lie in a metric space $(\mcX, c)$ with a unit diameter, i.e., $c(a,b) \le 1$ for any pair $(a,b) \in \mcX\times \mcX$ and the cost of moving unit mass from $a$ to $b$ is given by $d(a,b) = c(a,b)^p$. For any transport plan $\gamma$ between $\mu$ and $\nu$, the cost of $\gamma$ is defined as \[w_p(\gamma):=\left(\int_{\mcX\times \mcX}c(x,y)^p\,\mathrm{d}\gamma(x,y)\right)^{1/p}.\]Let $\gamma^*$ be a minimum-cost transport plan between $\mu$ and $\nu$. Then,  
the \emph{$p$-Wasserstein distance} between $\mu$ and $\nu$ is defined as 
$W_p(\mu, \nu):=w_p(\gamma^*)$.

The $p$-Wasserstein distance is a powerful metric for measuring similarities between probability distributions.  
Due to its numerous mathematical properties, the $p$-Wasserstein distance has found diverse applications including in machine learning~\cite{chang2023csot, chuang2022robust, esfahani,janati2019wasserstein, luise2018differential,oquab2023dinov2,  vincent2021semi}, computer vision~\cite{indykicml,gupta2010sparse, lai2022sar}, and natural language processing~\cite{alvarez2018gromov,huang2016supervised, yurochkin2019hierarchical}. 
One can estimate the $p$-Wasserstein distance between two unknown distributions $\mu$ and $\nu$ by simply taking $n$ samples from each $\mu$ and $\nu$ and then computing the $p$-Wasserstein distance between the discrete distributions over these samples (each sample point is assigned a mass of $1/n$). For $p \in [1,\infty)$, it is well-known that as $n\rightarrow \infty$, this \emph{empirical $p$-Wasserstein} distance converges to the true $p$-Wasserstein distance. Due to this law of weak convergence, the $p$-Wasserstein distance is used as a loss function in training generative models~\cite{arjovsky2017wasserstein, genevay, salimans}.

The $p$-Wasserstein distance is sensitive to geometric dissimilarities between the distributions.
Consider two distributions $\mu$ and $\nu = (1-\delta)\mu+\delta \nu'$ that differ only by a mass of $\delta$. The $p$-Wasserstein distance between $\mu$ and $\nu$ can be as high as $\delta^{1/p}W_p(\mu,\nu')$. Thus, as $p$ increases, the $W_p(\mu,\nu)$ increases by a rate of $\delta^{1/p}$, making $W_p$ more sensitive to such differences for larger values of $p$. The higher sensitivity of $p$-Wasserstein distance for $p>1$ makes it an attractive choice as a dissimilarity metric between distributions. Consequently, it can be used in 
clustering~\cite{el2020decwa, zhuang2022wasserstein} and barycenter computation~\cite{claici2018stochastic,cuturi2014fast, vaskevicius2023computational}.

The higher sensitivity of $p$-Wasserstein distance for larger values of $p$ also makes it susceptible to noise of two types: outliers and sampling discrepancy. Consider $\mu$ and $\nu = 0.99\mu + 0.01\nu'$ and $W_p(\mu,\nu')=1$, i.e., we add an \emph{outlier} mass of $\delta=0.01$ that is placed at a distance $1$ from $\mu$. In this case, $\mu$ and $\nu$ differ in only $1\%$ of mass and yet, the distance between $\mu$ and $\nu$ is $0.1$ when $p=2$, $0.21$ when $p=3$, and $1$ when $p=\infty$. Thus, for $p \ge 2$, outliers can disproportionately increase the distance between distributions.

Similar to outliers, sampling discrepancies in empirical distributions can also contribute disproportionately to the overall $p$-Wasserstein distance. As a result, in $2$-dimensions, the convergence rate of the empirical $p$-Wasserstein distance to the true distance drops to $n^{-1/2p}$ and for $p =\infty$, the empirical distance does not even converge to the real one~\cite{fournier2015rate}. To understand this phenomenon better, consider $p=2$ and a discrete distribution $\mu$ having two points $a$ and $b$ in its support, each assigned a probability mass of $1/2$. Let $c(a,b) =1$. Consider now two sets $X$ and $Y$ of $n$ samples drawn from $\mu$ and let $\mu_X$ (resp. $\mu_Y$) be the discrete distributions with points of $X$ (resp. $Y$) as the support and a mass of $1/n$ at each point in the support. Note that $\mathbb{E}[|\mu_X(a)-\mu_Y(a)|]=\Theta(1/\sqrt{n})$ and therefore, $W_2(\mu_X,\mu_Y)\approx n^{-1/4}$ and $W_p(\mu_X, \mu_Y) \approx n^{-1/2p}$. Thus, the rate of convergence for $p \ge 2$ is slower than for the case with $p=1$. Therefore, one needs significantly more samples to get an accurate estimate of the true $2$-Wasserstein distance. This restricts the use of $2$-Wasserstein distance (and also other higher values of $p$) as a loss function in learning tasks.

One way to overcome the impact of noise from outliers or sampling discrepancy is by using the partial $p$-Wasserstein distance.  
For \emph{$\alpha$-partial $p$-Wasserstein distance}, one wishes to compute the cheapest cost of a transport plan that transports $\alpha$ mass between distributions $\mu$ and $\nu$. Such transport plan is referred to as \emph{$\alpha$-optimal partial transport plan} (or simply \emph{$\alpha$-partial OT plan}). Given two distributions $\mu$ and $\tilde\nu = (1-\delta)\nu+\delta \nu'$, and under reasonable assumptions on the outlier distribution $\nu'$, one can show that the transport plan associated with 
a $(1-\delta)$-partial $p$-Wasserstein distance will transport mass only from the inliers. This observation was used to eliminate the impact of outliers in two distributions and applied to many ML tasks~\cite{choi2024generative,le2021robust, nietert2023outlier}. Most of these applications assume that the value of $\delta$ is given; see~\cite{caffarelli2010free, chapel2020partial, figalli2010optimal,nietert2022outlier}. Recently, \citet{phatak2022computing} introduced the idea of \emph{OT-profile}, which is a function that maps any $\alpha \in [0,1]$ to the $\alpha$-partial $p$-Wasserstein distance between $\mu$ and $\nu$. They showed that this function is a non-decreasing function\footnote{Their function maps $\alpha$ to the $p$th power of the $\alpha$-partial $p$-Wasserstein distance. In this paper, however, we assume that the function maps $\alpha$ to the partial $p$-Wasserstein distance and not its $p^{th}$ power.}, which can be used to also identify the value of $\delta$. 
All existing works that use partial $p$-Wasserstein distance to identify outliers are described for pairs of distribution. It is not clear how one can apply this distance on a set containing noisy distributions.

Additionally, there are two major drawbacks of using $(1-\delta)$-partial $p$-Wasserstein distance as a dissimilarity measure on sets of probability distributions. 
\begin{itemize}
    \item The $(1-\delta)$-partial $p$-Wasserstein distance does not satisfy the triangle inequality, and
    \item For two distributions $\mu$ and $\nu$ that differ by a mass less than $\delta$, the $(1-\delta)$-partial $p$-Wasserstein distance will be $0$, i.e., this cost is not sensitive to minor geometric differences in distributions.
\end{itemize}

In another line of work, given a parameter $\lambda > 0$, \citet{mukherjee2021outlier} presented a robust distance called the $\lambda$-ROBOT, which is simply the $p$-Wasserstein cost between $\mu$ and $\nu$ under the truncated ground distance metric $c_\lambda(a,b) = \min\{c(a,b), 2\lambda\}$
\footnote{Originally, \citet{mukherjee2021outlier} presented $\lambda$-ROBOT as the $1$-Wasserstein distance between $\mu$ and $\nu$ under $c_\lambda(\cdot,\cdot)$. For any $p>1$, one can extend their distance by computing the $p$-Wasserstein distance under $c_\lambda(\cdot,\cdot)$.}. Although $\lambda$-ROBOT is a metric, it remains sensitive to outliers and sampling discrepancies. For instance, similar to the $p$-Wasserstein distance, a mass of $\delta$ can disproportionately increase $\lambda$-ROBOT by $2\lambda\delta^{1/p}$. Also, the convergence rate of the empirical $\lambda$-ROBOT to the true $\lambda$-ROBOT in two-dimensional space would be $2\lambda n^{-1/2p}$.

An important open question is the following:
\vspace{0.5em}

{\it Can we design a new \underline{metric} that, for $p > 1$, retains the sensitivity of $p$-Wasserstein distance to minor geometric differences in the distributions, but is robust to noise?} 

\vspace{0.5em}

\paragraph{Our Results:} For any $k \ge 0$, we introduce a partial $p$-Wasserstein distance-based metric called $(p,k)$-RPW and we denote it by $\ours{p,k}{\cdot,\cdot}$. Our distance is simply the smallest $\varepsilon$ such that the $(1-\varepsilon)$-partial $p$-Wasserstein distance is at most $k\varepsilon$.

\begin{figure}
    \centering
    \begin{tabular}{c@{\hskip 2em}c}
         \includegraphics[width=0.4\linewidth]{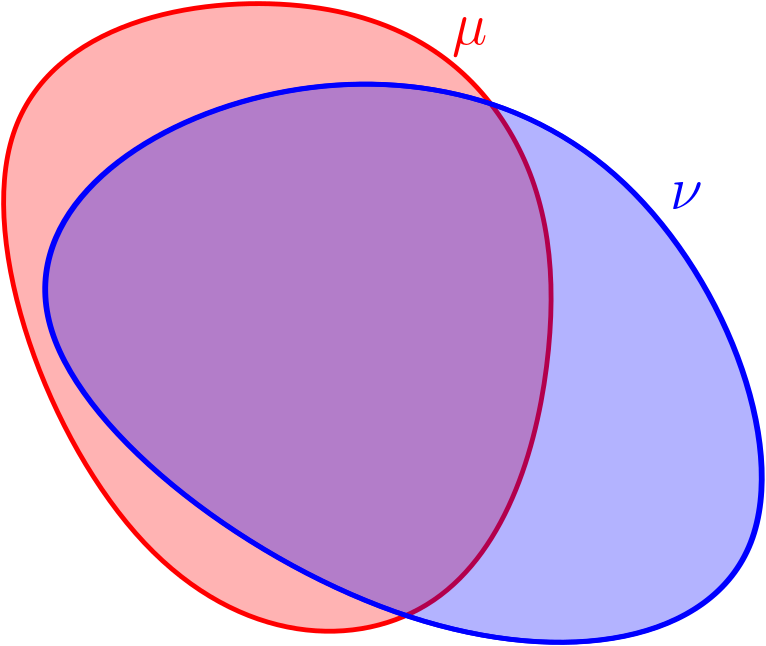}&\includegraphics[width=0.4\linewidth]{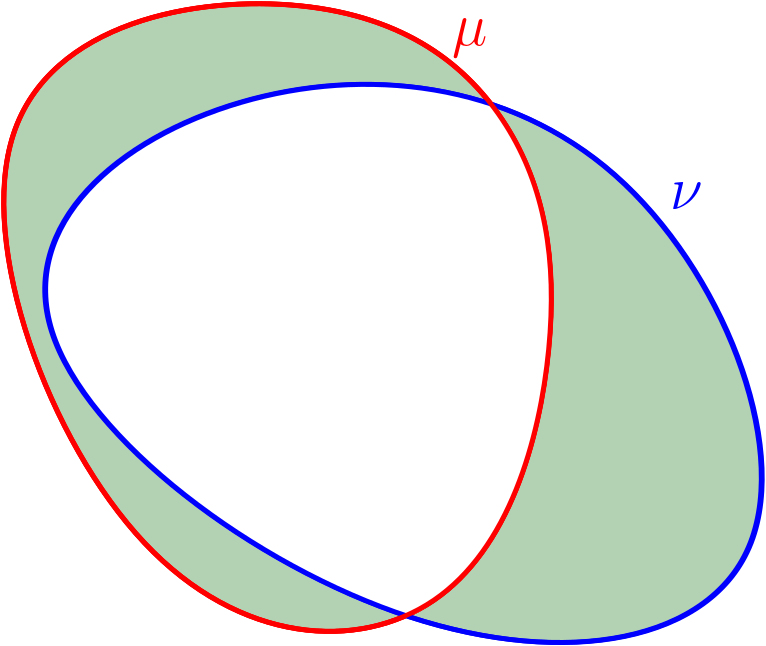}  \\[2pt]
         (i) Distributions $\mu$ and $\nu$ & (ii) $\|\mu-\nu\|_{\mathrm{TV}}$  \\[8pt]
         \includegraphics[width=0.4\linewidth]{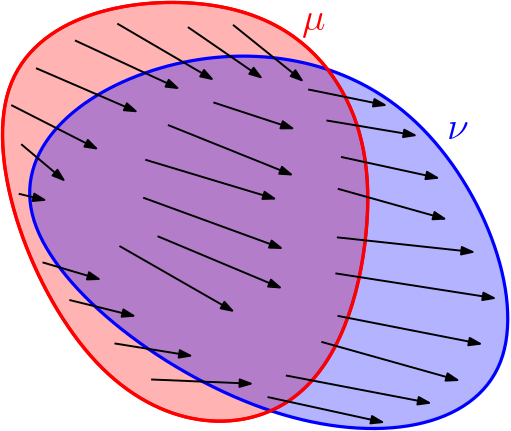}&\includegraphics[width=0.4\linewidth]{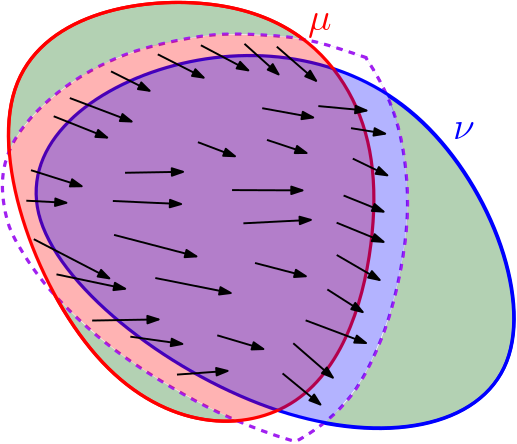}  \\
         (iii) $W_p(\mu, \nu)$ & (iv) $\ours{p,k}{\mu,\nu}$
    \end{tabular}
    \caption{Interpretations of different distance functions.}
    \label{fig:enter-label}
\end{figure}
 
Our distance combines the total variation distance with the $p$-Wasserstein distance. Recollect that the total variation distance between $\mu$ and $\nu$ is the mass that remains after all the co-located mass is transported. In Figure~\ref{fig:enter-label} (ii), the mass inside the green region represents the TV distance between two distributions. The $p$-Wasserstein distance measures the cost of the cheapest transport plan that leaves no mass behind. In Figure~\ref{fig:enter-label} (iii), the cost of moving all the mass from the red region to the blue region represents the $p$-Wasserstein distance. Our distance balances the two, i.e., we find an $\varepsilon$ such that a transport plan that leaves $\varepsilon$ mass behind has a cost of $k\varepsilon$. In Figure~\ref{fig:enter-label} (iv), our distance balances the cost of moving mass from the red region to the blue region with the mass remaining inside the green region. The robustness of our distance follows from the observation that noisy mass will be part of the green region (i.e., mass that is not transported) and therefore, cannot contribute disproportionately to the cost. 
We establish the following properties for our distance function:
\begin{itemize}
\item {\it Metric Property:} For any choice of $p \ge 1$ and $k \ge 0$, the distance $(p,k)$-RPW is a metric. Unlike the $(1-\delta)$-partial $p$-Wasserstein distance, our distance function satisfies the triangle inequality. Furthermore, unlike the $(1-\delta)$-partial $p$-Wasserstein distance, where two distributions $\mu$ and $\nu$ can have a cost of $0$ even if they differ by a mass of $\delta$, for any two distributions $\mu$ and $\nu$ with $\mu\neq \nu$, $\ours{p,k}{\mu,\nu} > 0$. See Theorem~\ref{lemma:metric}.
\item {\it Robust to Outliers:} Given two distributions $\mu$ and $\nu$, adding a mass of $\delta$ to $\nu$ will change $\ours{p,k}{\mu,\nu}$ by at most $\pm \delta$. In other words, an outlier mass of $\delta=0.01$ cannot increase the $\ours{p,k}{\mu,\nu}$ by more than $0.01$. Recollect that this can be as high as $0.1$ for $2$-Wasserstein distance and $0.21$ for $3$-Wasserstein distance. See Theorem~\ref{lemma:distortion}. 
\item {\it Robust to Sampling Discrepancy:} In $2$ dimensions, the $(p,1)$-RPW between empirical distributions converges to the true $(p,1)$-RPW distance at a rate of $n^{-\frac{p}{4p-2}}$. In contrast, the rate of convergence of the $2$-Wasserstein and the $\lambda$-ROBOT distances are $n^{-1/2p}$ and $2\lambda n^{-1/2p}$, respectively. Note that, for $p=\infty$, our distance converges at the rate of $n^{-1/4}$ whereas the $\infty$-Wasserstein distance does not converge. Our results extend to any dimension. For $d\ge 2$ and $p > \frac{d}{2}$, we show that the convergence rate of the empirical $(p,1)$-RPW is significantly faster than that of the $p$-Wasserstein distance. See Theorem~\ref{cor:convergence}.  
\end{itemize}

Alternatively, in Figure~\ref{fig:OTprofile}, suppose point $(x^*,y^*)$ is the intersection point of the line $y=k(1-x)$ with the OT-profile. Then, our distance is simply $(1-x^*)$. Note that when $k=0$, our distance becomes the total variation distance. When we set $k$ to be sufficiently large, our distance becomes $W_p(\mu,\nu)/k$. In this sense, our distance interpolates between the total variation distance and the $p$-Wasserstein distance.    
By choosing the parameters $k$ or $p$ correctly, we can reduce our distance to several well-known distances. 

\begin{figure}
    \centering
    \includegraphics[width=0.9\linewidth]{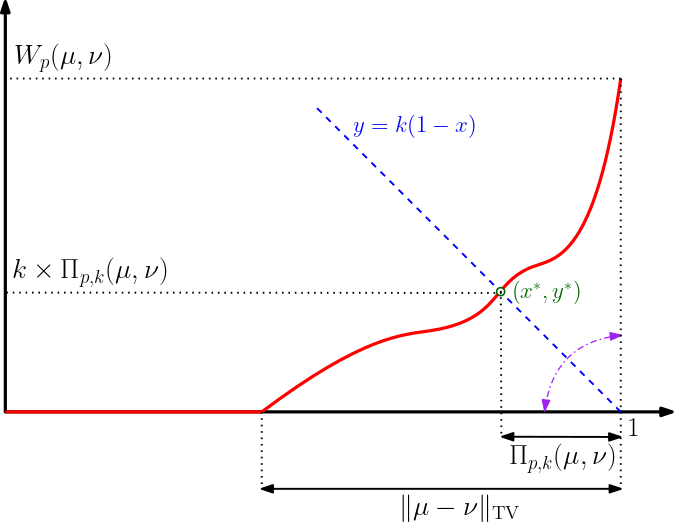}
    \vspace{-1em}
    \caption{Interpretation of distances based on the OT-profile.}
    \label{fig:OTprofile}
\end{figure}

\begin{itemize}
\item {\it Relation to \levy distance:} $(\infty, 1)$-RPW between any two distributions $\mu$ and $\nu$ is equal to their \levy distance. See Lemma~\ref{lemma:relationLevy}.
\item {\it Relation to total variation distance:} For any $p\ge1$, $(p,0)$-RPW between any two distributions $\mu$ and $\nu$ is equal to the total variation distance between $\mu$ and $\nu$. See Lemma~\ref{lemma:relationTV}.
\item {\it Relation to $p$-Wasserstein distance:} For any $p\ge1$, as $k \rightarrow \infty$, $k\times \ours{p,k}{\mu,\nu}$ approaches the $p$-Wasserstein distance. See Lemma~\ref{lemma:relationWasserstein}. 
\end{itemize}

In our experiments, we use our distance from a query image to rank images in a database of noisy images. Using this, we identify the top ten images in our database that are similar to the query. For the MNIST, CIFAR-10, and COREL datasets, our distance produces a higher accuracy in comparison to the accuracy produced by the $1$-Wasserstein, $2$-Wasserstein, and the TV distances.

\subsection{Notations.}
For any distribution $\mu$ defined over a compact set $\mcX$, let $\mass{\mu}:=\int_\mcX\mathrm{d}\mu(x)$ denote the total mass of $\mu$. 
For a metric space $(\mcX, c)$, define the diameter of $\mcX$ as $\max_{(a,b)\in \mcX\times\mcX}c(a,b)$.
For any pair of distributions $\mu$ and $\nu$ defined over $(\mcX, c)$ and parameters $p\ge 1$ and $\alpha\in[0,1]$, let $W_{p, \alpha}(\mu, \nu)$ denote the $\alpha$-partial $p$-Wasserstein distance between $\mu$ and $\nu$.

\section{Robust Partial \texorpdfstring{$p$}{}-Wasserstein Metric}
\label{sec:distance}
Given two probability distributions $\mu$ and $\nu$ defined over a metric space $(\mcX, c)$ with a unit diameter and any parameters $p\ge 1$ and $k\ge 0$, we define the $(p,k)$-\emph{Robust Partial $p$-Wasserstein distance} or simply $(p,k)$-\emph{RPW} between $\mu$ and $\nu$, denoted by $\ours{p,k}{\mu, \nu}$, to be the minimum value $\varepsilon\ge0$ such that the $(1-\varepsilon)$-partial $p$-Wasserstein distance between $\mu$ and $\nu$ is at most $k\varepsilon$; more precisely,
\begin{equation}\label{eq:ours}
    \ours{p,k}{\mu, \nu} = \inf\{\varepsilon\ge 0\mid W_{p, 1-\varepsilon}(\mu, \nu) \le k\varepsilon\}.
\end{equation}
Alternatively, let $P=(x^*, y^*)$ be the intersection point of the OT-profile curve with the line $y=k(1-x)$. Then, $(p,k)$-RPW between $\mu$ and $\nu$ would be $\ours{p,k}{\mu, \nu} = 1-x^*$. 

We show that $(p,k)$-RPW distance satisfies all the metric properties. The triangle inequality is the only property for which the proof is non-trivial. We provide a sketch of the proof below; see Appendix~\ref{sec:appendix-2} for details.

For any three probability distributions $\mu$, $\nu$, and $\kappa$, suppose $\ours{p,k}{\mu, \kappa}=\varepsilon_{1}$ and $\ours{p,k}{\kappa, \nu} =\varepsilon_{2}$. 
Let $\gamma_{1}$ denote a $(1-{\varepsilon_1})$-partial OT plan from $\mu$ to $\kappa$ and $\gamma_2$ be a $(1-\varepsilon_2)$-partial OT plan from $\kappa$ to $\nu$. In Figure~\ref{fig:metric}, the blobs in the left, middle, and right show the distributions $\mu$, $\kappa$, and $\nu$, respectively, and the blue (resp. red) arrows correspond to the transport plan $\gamma_1$ (resp. $\gamma_2$). Let $\kappa_{1}$ (resp. $\kappa_2$) be the mass of $\kappa$ that is transported from $\mu$ (resp. to $\nu$) by $\gamma_1$ (resp. $\gamma_2$) (shown in Figure~\ref{fig:metric} by the blue (resp. red) region inside the distribution $\kappa$). Define $\kappa_{\mathrm{c}}$ to be the distribution of mass of $\kappa$ that is common to both $\kappa_1$ and $\kappa_2$ (the purple region inside the distribution $\kappa$ in Figure~\ref{fig:metric}).
Note that the total mass of $\kappa_1$ that is not transported by $\gamma_2$ is at most $\varepsilon_2$; therefore,
\begin{equation}\label{eq:metric_0}
    \mass{\kappa_c}\ge \mass{\kappa_1}-\varepsilon_2= 1-\varepsilon_1-\varepsilon_2.
\end{equation}
Define $\mu_{\mathrm{c}}$ (resp. $\nu_{\mathrm{c}}$) to be the distribution whose mass is transported to (resp. from) $\kappa_c$ in $\gamma_1$ (resp. $\gamma_2$). In Figure~\ref{fig:metric}, the distribution $\mu_{\mathrm{c}}$ (resp. $\nu_{\mathrm{c}}$) is depicted by the purple region inside distribution $\mu$ (resp. $\nu$). From Equation~\eqref{eq:metric_0},
\begin{equation}\label{eq:metric_00}
    \mass{\mu_{\mathrm{c}}}=\mass{\nu_{\mathrm{c}}} = \mass{k_{\mathrm{c}}} \ge 1-\varepsilon_1-\varepsilon_2
\end{equation}
Therefore, we have
\begin{align}
    W_{p, 1-\varepsilon_1-\varepsilon_2}(\mu, \nu) &\le W_{p}(\mu_{\mathrm{c}}, \nu_{\mathrm{c}})\nonumber \\ &\le W_{p}(\mu_{\mathrm{c}}, \kappa_{\mathrm{c}}) + W_{p}(\kappa_{\mathrm{c}}, \nu_{\mathrm{c}})\nonumber \\ &\le k(\varepsilon_1 + \varepsilon_2).\label{eq:metric_4}
\end{align}
The second inequality follows from the triangle inequality of $p$-Wasserstein distances and the third inequality follows from the definition of our distance. Furthermore, 
since $\ours{p,k}{\mu,\nu}$ is the minimum $\varepsilon$ with $W_{p,1-\varepsilon}(\mu,\nu)\le k\varepsilon$, from Equation~\eqref{eq:metric_4}, $\ours{p,k}{\mu, \nu}\le   \varepsilon_{1} + \varepsilon_{2}$,
as desired.
\begin{figure}
    \centering
    \includegraphics[width=0.9\linewidth]{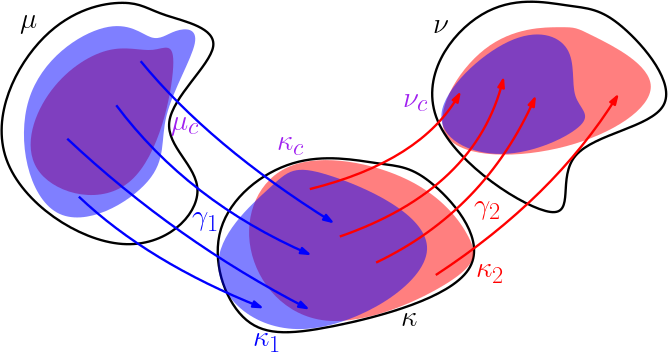}
    \vspace{-0.5em}
    \caption{The triangle inequality of the RPW distance function.}
    \label{fig:metric}
\end{figure}

\begin{restatable}{theorem}{metric}
Given a metric space $(\mcX, c)$ with a unit diameter and any parameters $p\ge 1$ and $k\ge 0$, the $(p,k)$-RPW distance function $\ours{p, k}{\cdot, \cdot}$ for all probability distributions defined over $(\mcX, c)$ is a metric.
\label{lemma:metric}
\end{restatable}

The following lemma highlights a useful feature of our metric, which is used in deriving its important properties.

\begin{restatable}{lemma}{maxalphabeta}\label{lemma:max_alpha_beta}
    Given two probability distributions $\mu$ and $\nu$ defined over a metric space $(\mcX, c)$ with a unit diameter and parameters $p\ge 1$ and $k\ge0$, suppose $W_{p,1-\alpha}(\mu, \nu)=k\beta$ for some $\alpha, \beta\ge 0$. Then, $\ours{p,k}{\mu, \nu}\le \max\{\alpha, \beta\}$. Furthermore, if $k\neq0$, then $\ours{p,k}{\mu, \nu}\le \min\{\alpha, \beta\}$.
\end{restatable}

\section{Robustness Properties}\label{sec:robustness}
In Section~\ref{sec:outlier}, we show that an outlier mass of $\delta$ cannot increase the RPW distance by more than $\delta$, i.e., the RPW distance is robust to outliers.  
In Section~\ref{sec:convergence}, we show that the rate of convergence of the empirical RPW distance to the real RPW distance is asymptotically smaller than the rate of convergence for $p$-Wasserstein distance. Thus, we show that the RPW distance is more robust to outliers as well as sampling discrepancies than the $p$-Wasserstein distance.

\subsection{Robustness to Outlier Noise}\label{sec:outlier}
For $\delta\in(0,1)$, let $\tilde{\nu}:=(1-\delta)\nu + \delta\nu'$ be a noisy distribution obtained from $\nu$ contaminated with $\delta$ mass from a noise distribution $\nu'$. In Theorem~\ref{lemma:distortion}, we show that by distorting the noise distribution $\nu'$, an adversary cannot arbitrarily change the $(p,k)$-RPW distance between $\mu$ and $\tilde{\nu}$.

\begin{restatable}{theorem}{distortion}\label{lemma:distortion}
    For any probability distributions $\mu, \nu$, and $\nu'$ defined over a metric space $(\mcX, c)$ with a unit diameter and parameters $p\ge 1, k\ge0$, and $\delta\in(0,1)$, let $\tilde{\nu}=(1-\delta)\nu + \delta\nu'$. Then, 
    \[\ours{p,k}{\mu,\nu} - \delta\le \ours{p,k}{\mu, \tilde{\nu}}\le (1-\delta)\ours{p,k}{\mu,\nu} + \delta.\]
\end{restatable}

For a distribution $\mu$ and a noisy distribution $\tilde\mu$ that differs from $\mu$ by only a $\delta$ fraction of mass (i.e., $\|\mu-\tilde\mu\|_{TV}=\delta$), consider the following assumption:
\begin{itemize}
    \item[(A1)] The $(1-\frac{\delta}{10})$-partial $p$-Wasserstein distance between $\mu$ and $\tilde\mu$ is at least $\frac{1}{2}W_p(\mu, \tilde\mu)$.
\end{itemize}
Assuming (A1), in the following lemma, we show that the $(p,k)$-RPW distance between $\mu$ and $\tilde\mu$ is proportionate to $\min\left\{\delta, \frac{1}{k}W_p(\mu, \tilde\mu) \right\}$.

\begin{restatable}{lemma}{robustnesss}\label{lemma:robustness}
    For a probability distribution $\mu$ defined over a metric space $(\mcX, c)$ with a unit diameter and $\delta>0$, let $\tilde\mu$ be a probability distribution that differs from $\mu$ by a $\delta$ fraction of mass satisfying assumption (A1). Then, for any parameters $p\ge 1$ and $k>0$,
    \begin{align*}
        \ours{p,k}{\mu, \tilde{\mu}} = \Theta\left(\min\left\{\delta, \frac{1}{k}W_p(\mu, \tilde\mu) \right\}\right).
    \end{align*}
\end{restatable}

Intuitively, if $\tilde\mu$ is only slightly different from $\mu$, i.e., $W_p(\mu, \tilde\mu)\le k\delta$, then the sensitivity of our metric would be similar to that of the $p$-Wasserstein distance. On the other hand, if $\tilde\mu$ is far from $\mu$ (i.e., the $\delta$ fraction of the mass of $\tilde\mu$ that is different from $\mu$ is an outlier noise and disproportionately increases the $p$-Wasserstein distance), then the sensitivity of $(p,k)$-RPW is bounded by $\delta$.

\subsection{Robustness to Sampling Discrepancies.}\label{sec:convergence}
Next, we show that in the $2$-dimensional Euclidean space, the rate of convergence of the empirical $(p,1)$-RPW to the true distance is $\tilde{O}(n^{-\frac{p}{4p-2}})$, which is significantly faster than the convergence rate of $\tilde{O}(n^{-\frac{1}{2p}})$ of the empirical $p$-Wasserstein distance~\cite{fournier2015rate}\footnote{$\tilde{O}()$ hides $\mathrm{poly}(\log n)$ from the convergence rate.}. In particular, for $p=\infty$, the convergence rate of our metric is $\tilde{O}(n^{-\frac{1}{4}})$, whereas the empirical $p$-Wasserstein distance does not converge to the true distance. For simplicity in presentation, we restrict our analysis to $p=2$. Our bounds for any $p\ge 1$ and $d\ge 2$ are stated in Theorem~\ref{cor:convergence}, whose proof is provided in Appendix~\ref{sec:robustness-appendix}.

\begin{lemma}\label{lemma:convergence_2d}
    For any two probability distributions $\mu$ and $\nu$ defined over a metric space with a unit diameter, suppose $\mu_n$ and $\nu_n$ are two empirical distributions of $\mu$ and $\nu$, respectively. Then, with a high probability,
    \begin{equation*}
        |\ours{2,1}{\mu, \nu} - \ours{2,1}{\mu_n, \nu_n}| = \tilde{O}(n^{-\frac{1}{3}}).
    \end{equation*}
\end{lemma}

Note that for any pair of distributions $\mu$ and $\nu$ and their empirical distributions $\mu_n$ and $\nu_n$, by the triangle inequality, 
\[|\ours{2,1}{\mu, \nu}-\ours{2,1}{\mu_n, \nu_n}|\le \ours{2,1}{\mu, \mu_n}+\ours{2,1}{\nu, \nu_n}.\]
Therefore, to prove Lemma~\ref{lemma:convergence_2d}, we bound $(2,1)$-RPW of any distribution $\mu$ to its empirical distribution $\mu_n$ in the following lemma.
\begin{lemma}\label{lemma:convergence_2}
    Given a continuous probability distribution $\mu$ in the $2$-dimensional Euclidean space and an empirical distribution $\mu_n$ of $\mu$, 
    $\ours{2,1}{\mu, \mu_n} = \tilde{O}(n^{-\frac{1}{3}})$ with a high probability.
\end{lemma}

We begin by defining a set of notations that assist in proving Lemma~\ref{lemma:convergence_2}.
Let $\mcG$ be a grid with cell side length $n^{-\alpha}$ inside the unit square. For any cell $\cell\in\mcG$, let $\mu(\cell)$ denote the mass of $\mu$ inside $\cell$. Define the \emph{excess mass} of a cell $\cell$ as $\excess_{\mu}(\cell):=\max\{0, \mu(\cell)-\mu_n(\cell)\}$ and the excess of the grid $\mcG$, denoted by $\excess_{\mu}(\mcG)$, as the total excess of all cells of $\mcG$, i.e., $\excess_{\mu}(\mcG):=\sum_{\cell\in\mcG}\excess_{\mu}(\cell)$. When $\mu$ is clear from context, we simplify notation and denote the excess of $\mcG$ by $\excess(\mcG)$.

\begin{restatable}{lemma}{convergenceexcess}\label{lemma:convergence_excess}
    For any distribution $\mu$ inside the unit square, an empirical distribution $\mu_n$ of $\mu$, and a grid $\mcG$ with cell side length $n^{-\alpha}$, $\excess_\mu(\mcG)= \tilde{O}(n^{\alpha-\frac{1}{2}})$ with a high probability. 
\end{restatable}

To better explain our proof for Lemma~\ref{lemma:convergence_2}, we first present a weaker bound of $\tilde{O}(n^{-\frac{3}{10}})$. In Appendix~\ref{sec:convergence_appendix}, we improve our analysis and obtain a rate of $\tilde{O}(n^{-\frac{1}{3}})$. 

{\bf A weaker bound for Lemma~\ref{lemma:convergence_2}.} In the following, we construct a transport plan $\gamma$ that transports all except $\tilde{O}(n^{-\frac{3}{10}})$ mass with a cost of $\tilde{O}(n^{-\frac{3}{10}})$. Having computed such transport plan, we then use Lemma~\ref{lemma:max_alpha_beta} to conclude that $\ours{p,k}{\mu, \nu}=\tilde{O}(n^{-\frac{3}{10}})$. The details are provided next.

Define $\mcG_1$ (resp. $\mcG_2$) to be a grid with cell side length $O(n^{-\alpha_1})$ (resp. $O(n^{-\alpha_2})$) for $\alpha_1:=\frac{3}{10}$ (resp. $\alpha_2:=\frac{1}{5}$). 
The grids $\mcG_1$ and $\mcG_2$ are constructed in a way that their boundaries are aligned with each other.  
Let $\gamma_1$ be a partial transport plan that arbitrarily transports, for any cell $\cell\in\mcG_1$, a mass of $\min\{\mu(\cell), \mu_n(\cell)\}$ from $\mu_n$ to $\mu$. 
Define $\mu^1$ (resp. $\mu_n^1$) to be the distribution of mass of $\mu$ (resp. $\mu_n$) that is not transported by $\gamma_1$. 
Let $\gamma_2$ be a transport plan that transports, for any cell $\cell\in\mcG_2$, a mass of $\min\{\mu^1(\cell), \mu_n^1(\cell)\}$ from $\mu_n^1$ to $\mu^1$. 
Define $\gamma=\gamma_1+\gamma_2$. This completes the construction of $\gamma$. In the appendix, instead of two grids, we consider $O(\log\log n)$ grids and obtain the bound claimed in Lemma~\ref{lemma:convergence_2}.

The transport plan $\gamma$ transports as much mass as possible inside each cell of $\mcG_2$; therefore, the total mass that is not transported by $\gamma$ is equal to the excess of the grid $\mcG_2$, which from Lemma~\ref{lemma:convergence_excess} is
\begin{equation}\label{eq:convergence_proof_0}
    1-\mass{\gamma} = \excess(\mcG_2) = \tilde{O}(n^{\alpha_2-\frac{1}{2}}) = \tilde{O}(n^{-\frac{3}{10}}).
\end{equation}

\begin{figure}
    \centering
    \begin{tabular}{c@{\hskip 1em}c}
        \includegraphics[width=0.45\linewidth]{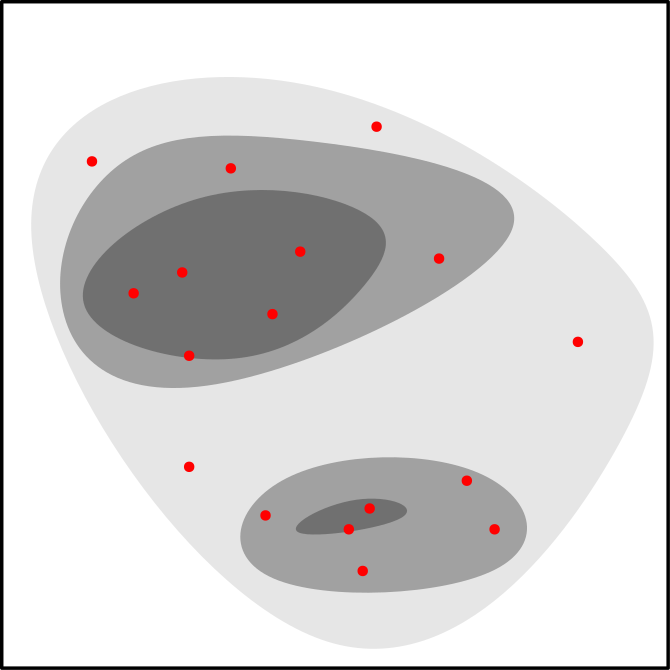} & \includegraphics[width=0.45\linewidth]{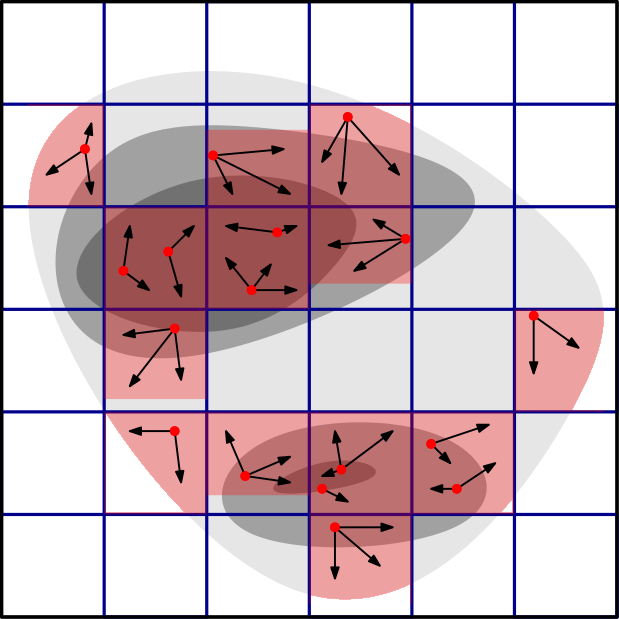} \\
        (a) & (b) \\[5pt]
        \includegraphics[width=0.45\linewidth]{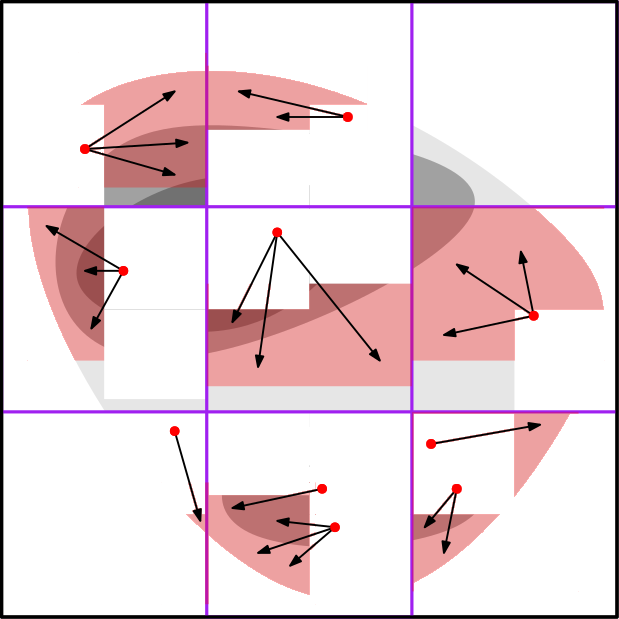} & \includegraphics[width=0.45\linewidth]{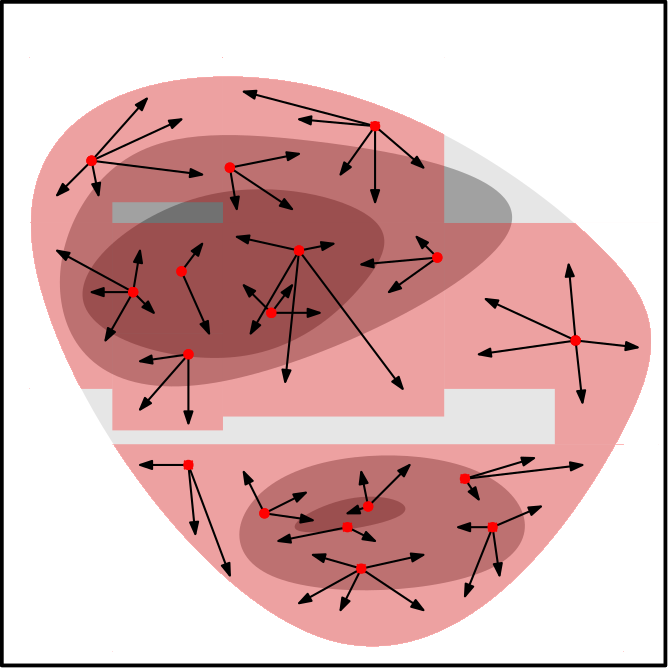} \\
        (c) & (d)
    \end{tabular}\\
    \caption{(a) A distribution $\mu$ (shaded gray area) and an empirical distribution $\mu_n$ (red dots), (b) $\gamma_1$ transports as much mass as possible inside each cell of $\mcG_1$, (c) for the remaining mass, $\gamma_2$ transports as much remaining mass as possible inside the cells of $\mcG_2$, and (d) the transport plan $\gamma$, which is the sum of $\gamma_1$ and $\gamma_2$.}
    \label{fig:convergence}
\end{figure}

Next, we show that the cost $w_2(\gamma)$ of the transport plan $\gamma$ is $\tilde{O}(n^{-\frac{3}{10}})$. Recall that $\gamma=\gamma_1+\gamma_2$. In our analysis, we first show that $w_2^2(\gamma_1)=O(n^{-\frac{3}{5}})$ and then show that $w_2^2(\gamma_2)=O(n^{-\frac{3}{5}})$. Using these two bounds, we then conclude that $w_2^2(\gamma)=O(n^{-\frac{3}{5}})$, or equivalently, $w_2(\gamma)=O(n^{-\frac{3}{10}})$.

Since $\gamma_1$ transports mass between points inside the same cell of $\mcG_1$, all mass transportation in $\gamma_1$ has a squared cost of $O(n^{-2\alpha_1}) = O(n^{-\frac{3}{5}})$ and $w_2^2(\gamma_1) = O(n^{-\frac{3}{5}})$. Furthermore, by Lemma~\ref{lemma:convergence_excess}, with a high probability, the total mass of $\mu^1$ is $\mass{\mu^1}=\tilde{O}(n^{\alpha_1-\frac{1}{2}})=\tilde{O}(n^{-\frac{1}{5}})$. Since the transport plan $\gamma_2$ transports at most $\mass{\mu^1}$ mass, each at a squared cost of $O(n^{-2\alpha_2}) = O(n^{-\frac{2}{5}})$, \[w_2^2(\gamma_2) = \tilde{O}\big(\mass{\mu^1}\times n^{-2\alpha_2}\big)=\tilde{O}(n^{-\frac{3}{5}}).\]
As a result, the cost of $\gamma$ is
\begin{align}
    w_2(\gamma) &= \sqrt{w_2^2(\gamma_1) + w_2^2(\gamma_2)} = \tilde{O}(n^{-\frac{3}{10}}).\label{eq:convergence_proof_1}
\end{align}
Note that $\gamma$ is a transport plan that transports all except $\alpha=\tilde{O}(n^{-\frac{3}{10}})$ mass (Equation~\eqref{eq:convergence_proof_0}) and has cost $\tilde{O}(n^{-\frac{3}{10}})$ (Equation~\eqref{eq:convergence_proof_1}). Hence, \[W_{2, 1-\alpha}(\mu, \mu_n)\le w_2(\gamma)=\tilde{O}(n^{-\frac{3}{10}}).\] Plugging into Lemma~\ref{lemma:max_alpha_beta}, $\ours{2,1}{\mu, \mu_n}=\tilde{O}(n^{-\frac{3}{10}})$.

In the appendix, we use an identical approach to obtain the rate of convergence for any dimension $d\ge1$ and any parameter $p\ge 1$. The following lemma summarizes the results. 

\begin{restatable}{lemma}{convergenceGeneral}\label{lemma:convergence}
    Given a continuous probability distribution $\mu$ in the $d$-dimensional Euclidean space, an empirical distribution $\mu_n$ of $\mu$, and parameters $p\ge 1$ and $k> 0$ constant, with a high probability,
    \begin{equation*}
        \ours{p,k}{\mu, \mu_n} = \begin{cases}
            \tilde{O}(n^{-\frac{1}{d}}), \qquad &p\le \frac{d}{2},\\
            \tilde{O}(n^{-\frac{p}{p(d+2)-d}}),  &p>\frac{d}{2}.
        \end{cases}
    \end{equation*}
\end{restatable}

For any pair of distributions $\mu$ and $\nu$ and their empirical distributions $\mu_n$ and $\nu_n$, by the triangle inequality, 
\[|\ours{p,k}{\mu, \nu}-\ours{p,k}{\mu_n, \nu_n}|\le \ours{p,k}{\mu, \mu_n}+\ours{p,k}{\nu, \nu_n}.\]
Combined with Lemma~\ref{lemma:convergence}, we get the following corollary.
\begin{theorem}\label{cor:convergence}
    For two probability distributions $\mu$ and $\nu$ defined over a metric space $(\mcX,c)$ with a unit diameter, suppose $\mu_n$ and $\nu_n$ are two empirical distributions of $\mu$ and $\nu$, respectively. Then, for any $p\ge 1$ and $k> 0$ constant, with a high probability,
    \begin{equation*}
        |\ours{p,k}{\mu, \nu} - \ours{p,k}{\mu_n, \nu_n}| = \begin{cases}
            \tilde{O}(n^{-\frac{1}{d}}), \qquad &p\le \frac{d}{2},\\
            \tilde{O}(n^{-\frac{p}{p(d+2)-d}}),  &p>\frac{d}{2}.
        \end{cases}
    \end{equation*}
\end{theorem}

\paragraph{Extension to arbitrary diameter.} We can extend our distance and its properties to the case where the diameter of the support is bounded by any $\Delta > 0$ as follows.  Define $(p,k)$-RPW between distributions $\mu$ and $\nu$ to be the minimum value $\varepsilon>0$ such that $W_{p, 1-\varepsilon}(\mu, \nu) \le k\Delta\varepsilon$.  In this case, the metric property (Theorem~\ref{lemma:metric}) and the robustness to outliers (Theorem~\ref{lemma:distortion}) holds without any changes, and given that the diameter $\Delta$ is a constant, our bounds for the convergence rate of the empirical $(p,k)$-RPW also holds as stated in Theorem~\ref{cor:convergence}.

\section{Relation to Other Distances}\label{sec:relation}
In this section, we discuss the relation of the $(p,k)$-RPW metric with three other well-known distance functions, namely (i) \levy distance, (ii) total variation, and (iii) $p$-Wasserstein distance. 
In particular, we first show that the $(\infty, 1)$-RPW is the same as the \levy distance. We next show that for any $p\ge 1$, the $(p,k)$-RPW metric is an interpolation between total variation and the $p$-Wasserstein distance. More precisely, $\ours{p,0}{\cdot, \cdot}$ is the same as the total variation distance, and for large values of $k$, the $(p,k)$-RPW will be close to $\frac{1}{k}W_p(\cdot, \cdot)$.

\paragraph{\levy distance.} For any two distributions $\mu$ and $\nu$ defined over a set $\mcX$ in the $d$-dimensional Euclidean space, let $\lp{\mu, \nu}$ denote the \levy distance of $\mu$ and $\nu$. In the following lemma, we show that the $(\infty, 1)$-RPW metric is equal to the \levy distance. The proof of this lemma, which is provided in Appendix~\ref{sec:relation-appendix}, is similar to the approach described by~\citet{lahn2021faster}.

\begin{restatable}{lemma}{relationLevyy}\label{lemma:relationLevy}
    For any pair of probability distributions $\mu$ and $\nu$ in a metric space $(\mcX, c)$ with a unit diameter, $\ours{\infty,1}{\mu, \nu}= \lp{\mu, \nu}$.
\end{restatable}

\paragraph{Total Variation.} For any pair of distributions $\mu$ and $\nu$, let $\|\mu-\nu\|_{\mathrm{TV}}$ denote the total variation of $\mu$ and $\nu$. In Lemma~\ref{lemma:relationTV}, we show that for any $p\ge 1$, the $(p,0)$-RPW distance between $\mu$ and $\nu$ is equal to their total variation.  
Intuitively, the $(p,0)$-RPW distance measures the maximum amount of mass that can be transported from $\mu$ to $\nu$ at a $0$ cost, i.e., $(p,0)$-RPW distance is the amount of mass of $\mu$ and $\nu$ that overlap, which is the same as their total variation.

\begin{restatable}{lemma}{relationTV}\label{lemma:relationTV}
    For any two probability distributions $\mu$ and $\nu$ in a metric space $(\mcX, c)$ with a unit diameter and any parameter $p\ge 1$, $\ours{p,0}{\mu, \nu}=\|\mu-\nu\|_{\mathrm{TV}}$.
\end{restatable}

\paragraph{$p$-Wasserstein distance.} Finally, we show that for large enough values of $k$, the $(p,k)$-RPW metric would be close to $\frac{1}{k}W_p(\mu, \nu)$.

\begin{restatable}{lemma}{relationWasserstein}\label{lemma:relationWasserstein}
For any two probability distributions $\mu$ and $\nu$ over a metric space $(\mcX, c)$ with a unit diameter and any parameters $p\ge 1$ and $k>0$, $\ours{p,k}{\mu, \nu} \le \frac{1}{k}W_p(\mu, \nu)\le\ours{p,k}{\mu, \nu} + k^{-\frac{p+1}{p}}$.
\end{restatable}

\section{Algorithms to Compute \texorpdfstring{$(p,k)$}{}-RPW}\label{sec:approx}
In this section, we describe two approximation algorithms for computing the $(p,k)$-RPW distance between two discrete distributions defined over supports of $n$ points. The first algorithm uses a binary search on the value of $\ours{p,k}{\mu, \nu}$ and computes a $\delta$-additive approximation of $(p,k)$-RPW (or simply \emph{$\delta$-close $(p,k)$-RPW}) for any $\delta>0$ in $O(n^3\log n\log\delta^{-1})$ time. Our second algorithm relies on the algorithm by \cite{our-neurips-2019-otapprox} (LMR algorithm) to approximate the OT-profile and computes a $\delta$-additive approximation of our metric in $O(\frac{n^{2}}{\delta^p} + \frac{n}{\delta^{2p}})$ time. A high-level overview of each algorithm is presented below. See Appendix~\ref{sec:approx-ap} for full details.  

Note that when $k=0$, as discussed in Lemma~\ref{lemma:relationTV}, the $(p,0)$-RPW is simply the total variation distance and can be computed on discrete distributions in linear time. Hence, in the following, we assume $k>0$.

{\bf Highly-Accurate Algorithm.} For any $\varepsilon\in[0,1]$, computing the $(1-\varepsilon)$-partial $p$-Wasserstein distance for discrete distributions can be done using a standard OT solver in an augmented space~\cite{chapel2020partial}, which takes $O(n^3\log n)$ time~\cite{edmonds1972theoretical, orlin1988faster}. Our first algorithm is based on this observation and uses a simple binary search on the value of $(p,k)$-RPW to obtain a $\delta$-additive approximation in $O(n^3\log n\log\delta^{-1})$ time.

{\bf Computing Through an Approximate OT-Profile.} We can also approximate the $(p,k)$-RPW distance by using the LMR algorithm \cite{our-neurips-2019-otapprox} to approximate the OT-profile. Given two discrete probability distributions, an error parameter $\delta'>0$, and any cost function, the LMR algorithm incrementally constructs a $(\delta')^{1/p}$-additive approximation of the OT-profile, i.e., for any $\alpha\in[0,1]$, the LMR algorithm computes a $(\delta')^{1/p}$-additive approximation of the $\alpha$-partial $p$-Wasserstein distance~\cite{phatak2022computing}. 

In Lemma~\ref{lemma:approx} in the appendix, we show that computing our metric using a $\delta'$-additive approximation of the OT-profile leads to a $\frac{2\delta'}{k}$-close $(p,k)$-RPW. Therefore, to compute a $\delta$-close $(p,k)$-RPW distance function, we use the LMR algorithm with an error parameter $\delta'=(\frac{k\delta}{2})^p$ to approximate the OT-profile and to compute a $\delta$-close $\ours{p,k}{\mu, \nu}$ in $O(\frac{n^{2}}{\delta^p} + \frac{n}{\delta^{2p}})$ time.

\section{Experimental Results}
\label{sec:experimental}
In this section, we present the results of our experiments showing that our distance is robust to noise from outliers and sample discrepancies. 

In our first experiment, we use the $1$-Wasserstein distance, $2$-Wasserstein distance, TV distance, $(2,1)$-RPW, and $(2,0.1)$-RPW to rank images from the MNIST~\cite{lecun1998mnist}, CIFAR-10~\cite{hinton2012improving}, and COREL datasets and measure the accuracy of the results.  In our second experiment, we measure the convergence rate of empirical $(2,k)$-RPW distance to the true $(2,k)$-RPW and compare it with the convergence rate for $2$-Wasserstein distance for synthetic data sets. 
For both experiments, we compute an additive approximation of the RPW metric using the LMR algorithm~\cite{our-neurips-2019-otapprox}.

\subsection{Image Retrieval.}
Following the experimental setup introduced by \citet{rubner2000earth}, we conduct experiments on retrieving images using $(2,1)$-RPW and $(2,0.1)$-RPW distances and compare their accuracy against the $1$-Wasserstein, $2$-Wasserstein, and TV distances. 
In this experiment, given a dataset of labeled images and a set of unlabeled query images, the goal is to retrieve, for each query image, a set of $m$ similar images from the labeled dataset. The accuracy of a distance function in the image retrieval task is then computed as the ratio of the retrieved images with the correct label, averaged over all retrievals for all query images. In our experiments, we vary the value of $m$ from $1$ to $100$. 

{\bf Datasets.} We conduct the experiments using three datasets, namely, MNIST, CIFAR-10, and COREL. In our experiments on each dataset, we randomly select $2k$ images as the labeled dataset and randomly select $50$ images as the query. For each dataset, we introduce three scenarios of perturbation: 
\begin{itemize}
    \item[(i)] (noise in datasets) In this scenario, we add random noise to the images in the datasets. For the MNIST dataset, for each image, we add a random amount of noise between $0\%$ and $10\%$ to a random pixel in the image. For CIFAR-10 and COREL datasets, we replaced randomly selected $10\%$ of pixels with white pixels in each image.
    \item[(ii)] (shift in datasets) In this scenario, we shift up each labeled image by $2$ pixels for the MNIST dataset and increase the intensity of a random RGB channel by $20$ in each image of the CIFAR-10 and COREL datasets.
    \item[(iii)] (noise and shift in datasets) In the last scenario, we introduce both random noise (scenario (i)) and random shift (scenario (ii)) to the images in datasets.
\end{itemize}

{\bf Results.} Figure~\ref{fig:experimental-results-IR-mnist} shows the results of our experiments on MNIST and CIFAR-10 datasets. The results of our experiments on the COREL dataset are provided in Appendix~\ref{sec:appendix-exp}. In Figure~\ref{fig:experimental-results-IR-mnist}, the left (resp. right) column corresponds to our experiments on the MNIST (resp. CIFAR-10) dataset in the three scenarios described above. In each plot, the horizontal axis is the number $m$ of retrieved images for each query, and the vertical axis is the accuracy of the retrieved images. 
The results of our experiments suggest that our metric performs better than the $1$-Wasserstein, $2$-Wasserstein, and the TV distances for the task of image retrieval with perturbations. 

\begin{figure}[t]
\centering
\begin{tabular}{c@{\hskip 5pt}c}
     \includegraphics[width=.45\linewidth]{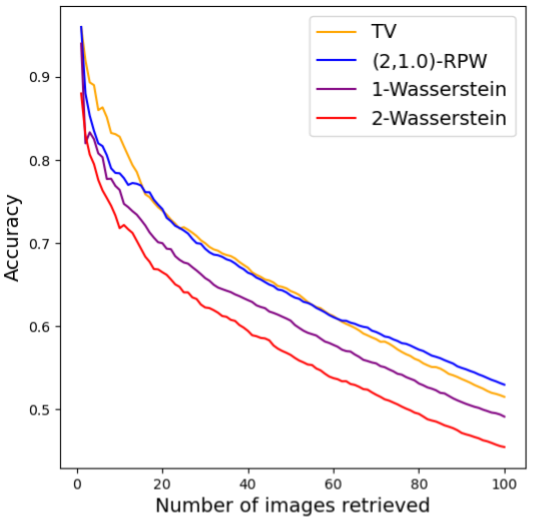} &  \includegraphics[width=.45\linewidth]{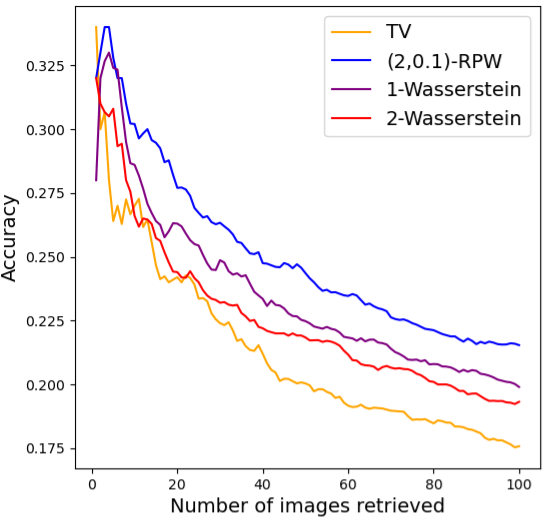}\\
     \multicolumn{2}{c}{\footnotesize (i) Noise in datasets}\\[10pt]
     \includegraphics[width=.45\linewidth]{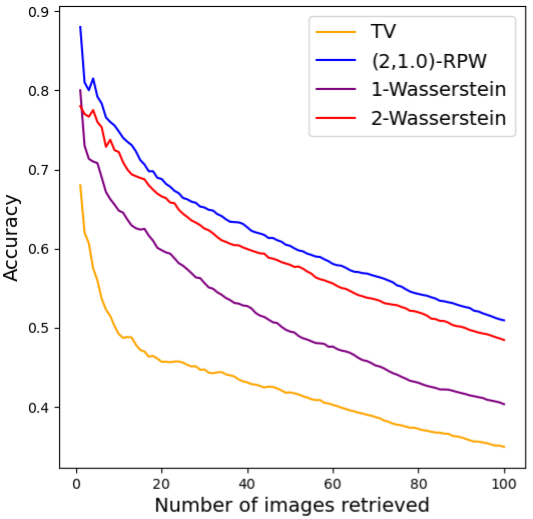} &  \includegraphics[width=.45\linewidth]{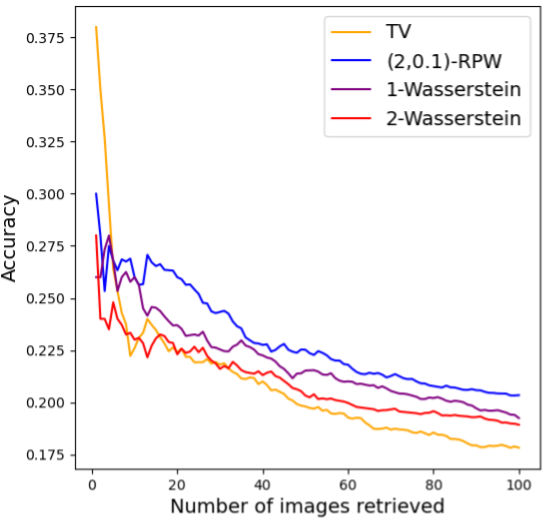}\\
     \multicolumn{2}{c}{\footnotesize (ii) Shift in datasets}\\[10pt]
     \includegraphics[width=.45\linewidth]{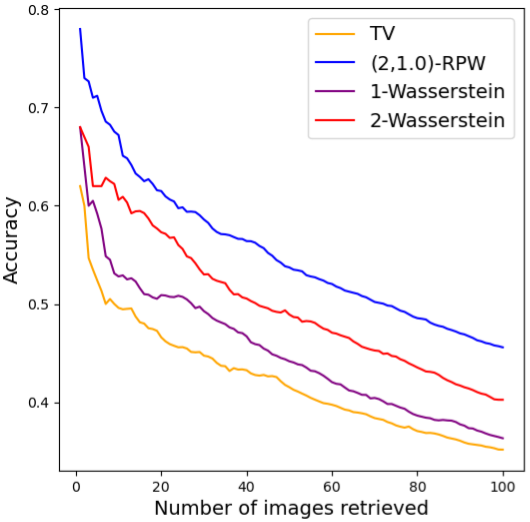} &  \includegraphics[width=.45\linewidth]{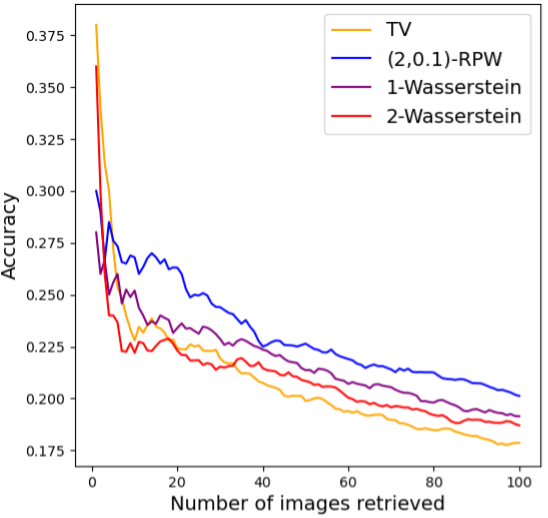}\\
     \multicolumn{2}{c}{\footnotesize (iii) Noise and shift in datasets}\\[8pt]
\end{tabular}
\vspace{-1em}
\caption{The results of our experiments on image retrieval on (left column) MNIST dataset and (right column) CIFAR-10 dataset.} 
\label{fig:experimental-results-IR-mnist}
\end{figure}

First, for experiments on the MNIST dataset, we observe that for datasets (ii) and (iii), the accuracy achieved by the $2$-Wasserstein distance is better than that of the $1$-Wasserstein distance. This, we believe is due to the higher sensitivity of $2$-Wasserstein distance, which enables it to detect minor differences effectively. However, for dataset (i), the presence of noise distorts $2$-Wasserstein distance, causing it to underperform. 

{\it Results for MNIST dataset (i):} The TV, $1$-Wasserstein, and $(2,1)$-RPW distances are more robust to noise and therefore, perform better than the $2$-Wasserstein distance. The $(2,1)$-RPW distance outperforms $1$-Wasserstein distance. This is because it retains the higher sensitivity of $2$-Wasserstein distance, which enables it to effectively detect minor differences.

{\it Results for MNIST dataset (ii):} TV distance is known to be sensitive to shifts. As a result, the accuracy of TV distance drops significantly. In contrast, the $2$-Wasserstein distance and the $(2,1)$-RPW distance handle such shifts more effectively. 

{\it Results for MNIST dataset (iii):} The $2$-Wasserstein distance is sensitive to noise and the TV distance is sensitive to shifts. Therefore, both these distances produce lower accuracy results. Recollect that the $(2,1)$-RPW distance combines the TV distance and the $2$-Wasserstein distance and therefore, can handle both shifts and noise effectively. Therefore, it outperforms all distances in this setting.

{\it Results for CIFAR-10 dataset: } In the CIFAR-10 dataset, images that have the same labels may already have significant variations and shifts. Due to these variations, $2$-Wasserstein and TV distances achieve lower accuracy in comparison to the $1$-Wasserstein distance. The $(2,0.1)$-RPW distance, however, outperforms the $1$-Wasserstein distance for this dataset.

\subsection{Rate of Convergence}
We conduct numerical experiments to compare the convergence rate of the empirical $(2,k)$-RPW metric with that of the $2$-Wasserstein distance and TV distance on discrete distributions. 
We compute the convergence rate of each metric by drawing two sets of $n$ i.i.d samples from a discrete distribution and compute the empirical distance between the corresponding empirical distributions.

{\bf Datasets.} We employ two synthetic $2$-dimensional discrete distributions, namely (i) (2-point distribution) a discrete distribution defined over $2$ points $a$ and $b$ each with a probability $1/2$, where $\|a-b\|=1$, and (ii) (grid distribution) a discrete distribution defined over $16$ points that are placed in a grid of $4\times 4$, where each point has a probability of $\frac{1}{16}$.

In our experiments, we vary the sample size $n$ from $10$ to $10^6$. For each value of $n$, we conduct the experiment $10$ times and take the mean distance among all $10$ executions.

{\bf Results.} As shown in Figure~\ref{fig:converge_rate}, experiments suggest that for both distributions, the empirical $(2,1)$-RPW distance converges to $0$ significantly faster than the $2$-Wasserstein distance. We also observe that for a small value of $k$ (e.g., $k=0.1$), $(2,k)$-RPW is close to the TV distance, whereas for a large value of $k$ (e.g., $k=10$), the $(2,k)$-RPW distance values are similar to $\frac{1}{k}W_2(\cdot,\cdot)$. These results are in line with our theoretical bounds in Section~\ref{sec:relation}.

\begin{figure}[t]
\centering
\includegraphics[width=.49\linewidth]{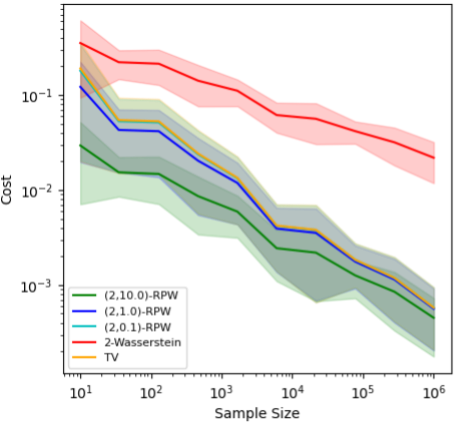}
\includegraphics[width=.49\linewidth]{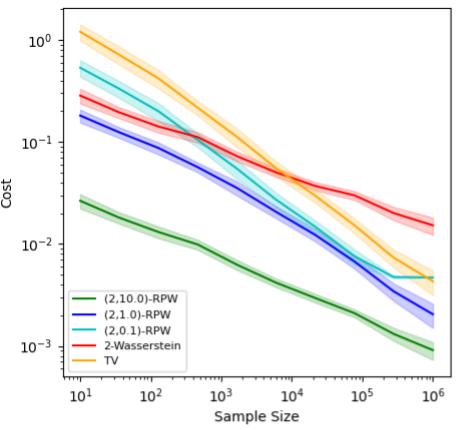}
\vspace{-1em}
\caption{The convergence rate of different metrics on (left) 2-point distribution and (right) grid distribution.}
\label{fig:converge_rate}
\end{figure}

\section{Conclusion}
In this paper, we designed a new partial $ p$-Wasserstein-based metric called the $(p,k)$-RPW that is robust to outlier noise as well as sampling discrepancies but retains the sensitivity of $p$-Wasserstein distance in capturing the minor geometric differences in distributions. We showed that our distance interpolates between the $p$-Wasserstein and TV distances and inherits robustness to both noise and shifts in distribution. We also showed that, for $p=\infty$, our metric is the same as the \levy distance.

The main contribution of this paper is to introduce a new metric and derive its useful properties. Experiments suggest that the new distance is a promising alternative to TV and $p$-Wasserstein distances. Including this distance into many potential machine learning use cases, including as a loss function in GANs or for computing barycenters of a set of distributions, remains an open question. Designing a parallel algorithm to approximate the $(p,k)$-RPW distance also remains an important open question.

\section*{Acknowledgement}
We would like to acknowledge Advanced Research Computing (ARC) at Virginia Tech, which provided us with the computational resources used to run the experiments. The research presented in this paper was funded by NSF CCF-2223871. We thank the anonymous reviewers for their useful feedback.

\section*{Impact Statement}
This paper presents work whose goal is to advance the field of 
Machine Learning. There are many potential societal consequences 
of our work, none which we feel must be specifically highlighted here.

\bibliography{main}
\bibliographystyle{icml2024}

\newpage
\appendix
\onecolumn
\section{Missing Proofs and Details}

\subsection{Missing Proofs of Section~\ref{sec:distance}.}\label{sec:appendix-2}
In this section, we provide the proofs for Theorem~\ref{lemma:metric} and Lemma~\ref{lemma:max_alpha_beta}.

\metric*
\begin{proof}
To prove this lemma, we show that $(p,k)$-RPW satisfies all four properties of the metric spaces, namely (i) identity, (ii) positivity, (iii) symmetry, and (iv) triangle inequality, and conclude that it is metric. 

For any probability distribution $\mu$ defined over $(\mcX,c)$, $W_p(\mu, \mu)=0$ and therefore, $\varepsilon=0$ satisfies the condition in Equation~\eqref{eq:ours}; hence, $\ours{p,k}{\mu, \mu}=0$ and property (i) holds. 
Furthermore, if $\nu$ is another probability distribution over $(\mcX,c)$ that is distinct from $\mu$, then  $W_p(\mu, \nu) > 0$. Hence, $\varepsilon=0$ does not satisfy the condition in Equation~\eqref{eq:ours}, and $\ours{p,k}{\mu, \nu}$, which is the smallest $\varepsilon\ge 0$ with $W_{p,1-\varepsilon}(\mu, \nu)\le k\varepsilon$, will be positive and property (ii) holds. 
Additionally, for any $\varepsilon\ge 0$, the $(1-\varepsilon)$-partial $p$-Wasserstein distance is symmetric, i.e., $W_{p, 1-\varepsilon}(\mu, \nu)=W_{p, 1-\varepsilon}(\nu, \mu)$. Therefore, from Equation~\eqref{eq:ours}, RPW is also symmetric, i.e., $\ours{p,k}{\mu, \nu} = \ours{p,k}{\nu, \mu}$ and property (iii) holds as well.

Finally, we show that the $(p,k)$-RPW satisfies the triangle inequality. For any three probability distributions $\mu$, $\nu$, and $\kappa$, suppose $\ours{p,k}{\mu, \kappa}=\varepsilon_{1}$ and $\ours{p,k}{\kappa, \nu} =\varepsilon_{2}$. In the following, we show that $\ours{p,k}{\mu, \nu}\leq \varepsilon_{1} + \varepsilon_{2}$ and conclude property (iv).

Note that if $\varepsilon_{1}+\varepsilon_{2} \ge 1$, then the triangle inequality holds trivially since $\ours{p,k}{\mu, \nu} \le 1 \le \varepsilon_{1} + \varepsilon_{2}$. Therefore, we assume that $\varepsilon_{1}+\varepsilon_{2} < 1$. Let $\gamma_{1}$ denote a $(1-{\varepsilon_1})$-partial OT plan from $\mu$ to $\kappa$ and define $\kappa_{1}:=\gamma_1\#\mu$ to be the mass of $\kappa$ that is transported from $\mu$ by $\gamma_1$; here, $\#$ denotes the push-forward operation. Similarly, let $\gamma_{2}$ denote a $(1-{\varepsilon_2})$-partial OT plan from $\kappa$ to $\nu$ and define $\kappa_{2}:=(\gamma_2)^{-1}\#\nu$ to be the mass of $\kappa$ that is transported to $\nu$ by $\gamma_2$. 
Then, both $\kappa_1$ and $\kappa_2$ are distributions over $\mcX$ that are dominated by $\kappa$ and have masses $\mass{\kappa_1}=1-\varepsilon_1$ and $\mass{\kappa_2}=1-\varepsilon_2$.

Define $\kappa_{\mathrm{c}}$ to be the distribution of mass of $\kappa$ that is common to both $\kappa_1$ and $\kappa_2$; more precisely, for each $x\in \mcX$,
\[\kappa_{\mathrm{c}}(x) = \min\{\kappa_1(x), \kappa_2(x)\}.\]
Note that the total mass of $\kappa_1$ that is not transported by $\gamma_2$ is at most $\varepsilon_2$; therefore,
\begin{equation}\label{eq:metric_0_a}
    \mass{\kappa_c}\ge \mass{\kappa_1}-\varepsilon_2= 1-\varepsilon_1-\varepsilon_2.
\end{equation}
Define $\mu_{\mathrm{c}}:=(\gamma_1)^{-1}\# \kappa_{\mathrm{c}}$ (resp. $\nu_{\mathrm{c}}:=\gamma_2\# \kappa_{\mathrm{c}}$) to be distribution dominated by $\mu$ (resp. $\nu$) whose mass is transported to (resp. from) $\kappa_c$ in $\gamma_1$ (resp. $\gamma_2$). From Equation~\eqref{eq:metric_0_a},
\begin{equation}\label{eq:metric_00_a}
    \mass{\mu_{\mathrm{c}}}=\mass{\nu_{\mathrm{c}}} = \mass{k_{\mathrm{c}}} \ge 1-\varepsilon_1-\varepsilon_2
\end{equation}
By the triangle inequality of the $p$-Wasserstein distance,
\begin{equation}\label{eq:metric_1_a}
    W_{p}(\mu_{\mathrm{c}}, \nu_{\mathrm{c}})\le W_{p}(\mu_{\mathrm{c}}, \kappa_{\mathrm{c}}) + W_{p}( \kappa_{\mathrm{c}}, \nu_{\mathrm{c}}).
\end{equation}
Furthermore, 
\begin{align}
    W_{p}(\mu_{\mathrm{c}}, \kappa_{\mathrm{c}}) \le w_{p}(\gamma_1)\le k\varepsilon_{1}.\label{eq:metric_2_a}
\end{align}
Similarly, 
\begin{align}
    W_{p}(\kappa_{\mathrm{c}}, \nu_{\mathrm{c}}) \le w_{p}(\gamma_2)\le k\varepsilon_{2}.\label{eq:metric_3_a}
\end{align}
Combining Equations~\eqref{eq:metric_00_a}, \eqref{eq:metric_1_a}, \eqref{eq:metric_2_a}, and~\eqref{eq:metric_3_a},
\begin{align}
    W_{p, 1-\varepsilon_1-\varepsilon_2}(\mu, \nu) &\le W_{p}(\mu_{\mathrm{c}}, \nu_{\mathrm{c}})\le W_{p}(\mu_{\mathrm{c}}, \kappa_{\mathrm{c}}) + W_{p}(\kappa_{\mathrm{c}}, \nu_{\mathrm{c}})\le k(\varepsilon_1 + \varepsilon_2).\label{eq:metric_4_a}
\end{align}
Since $\ours{p,k}{\mu,\nu}$ is the minimum $\varepsilon$ with $W_{p,1-\varepsilon}(\mu,\nu)\le k\varepsilon$, from Equation~\eqref{eq:metric_4_a}, $\ours{p,k}{\mu, \nu}\le   \varepsilon_{1} + \varepsilon_{2}$,
as desired.
\end{proof}

\maxalphabeta*
\begin{proof}
    Let $\delta:=\ours{p,k}{\mu, \nu}$. We prove this lemma by considering two cases:
    \begin{itemize}
        \item If $\alpha \le \beta$, then $W_{p, 1-\beta}(\mu, \nu)\le W_{p, 1-\alpha}(\mu, \nu) = k\beta$; hence, $\beta$ satisfies the condition in Equation~\eqref{eq:ours}, and $\ours{p,k}{\mu,\nu}\le \beta=\max\{\alpha, \beta\}$. Furthermore, if $k>0$, then $\delta\ge \alpha=\min\{\alpha, \beta\}$ because otherwise, if $\delta < \alpha$, then \[W_{p, 1-\delta}(\mu, \nu)\ge W_{p, 1-\alpha}(\mu, \nu)=k\beta\ge k\alpha > k\delta,\]
        which is a contradiction. 
    
        \item Otherwise, $\alpha > \beta$ and $W_{p, 1-\alpha}(\mu ,\nu) = k\beta<k\alpha$; hence, $\alpha$ satisfies the condition in Equation~\eqref{eq:ours},  and $\ours{p,k}{\mu,\nu}\le \alpha=\max\{\alpha, \beta\}.$ Additionally, if $k>0$, then $\delta\ge \beta=\min\{\alpha, \beta\}$, since otherwise, if $\delta < \beta$, then
        \begin{align*}
            W_{p,1-\delta}(\mu, \nu)&\ge W_{p,1-\beta}(\mu, \nu) \ge W_{p,1-\alpha}(\mu, \nu)=k\beta>k\delta,
        \end{align*}
        which is a contradiction.
    \end{itemize}
\end{proof}

\subsection{Missing Proofs and Details of Section~\ref{sec:robustness}}\label{sec:robustness-appendix}

\subsubsection{Robustness to Outlier Nosie}

\distortion*
\begin{proof}
    First, note that by the triangle inequality, $\ours{p,k}{\mu, \tilde\nu} + \ours{p,k}{\tilde\nu, \nu} \ge \ours{p,k}{\mu, \nu}$. Furthermore, by the definition of $\tilde\nu$, 
    $W_{p,1-\delta}(\nu, \tilde\nu) = 0$. Therefore, by Lemma~\ref{lemma:max_alpha_beta}, $\ours{p,k}{\nu, \tilde\nu}\le \max\{\delta, 0\} = \delta$ and
    \[\ours{p,k}{\mu, \tilde\nu}\ge \ours{p,k}{\mu, \nu}-\ours{p,k}{\nu, \tilde\nu}\ge \ours{p,k}{\mu, \nu} - \delta.\]
    Define $\alpha:=\ours{p,k}{\mu, \nu}$ and let $\gamma$ be a $(1-\alpha)$-partial OT plan from $\mu$ to $\nu$. Since $\tilde\nu$ is defined as $(1-\delta)\nu+\delta\nu'$, the transport plan $\gamma':=(1-\delta)\gamma$ can be seen as a $((1-\delta)(1-\alpha))$-partial transport plan from $\mu$ to $\tilde{\nu}$; therefore, \[W_{p, (1-\delta)(1-\alpha)}(\mu, \tilde{\nu})\le w_p(\gamma')= (1-\delta)^{\frac{1}{p}}w_p(\gamma)\le k(1-\delta)^{\frac{1}{p}}\alpha,\] where the last inequality holds by the definition of the $(p,k)$-RPW distance. Using Lemma~\ref{lemma:max_alpha_beta}, 
    \begin{align*}
        \ours{p,k}{\mu, \tilde{\nu}}&\le \max\left\{1-(1-\delta)(1-\alpha), (1-\delta)^{\frac{1}{p}}\alpha\right\}= (1-\delta)\alpha + \delta =  (1-\delta)\ours{p,k}{\mu,\nu} + \delta.
    \end{align*}
\end{proof}

The following lemma helps in proving Lemma~\ref{lemma:robustness}.

\begin{restatable}{lemma}{robustness}\label{lemma:robustness-app}
    For two probability distributions $\mu$ and $\tilde\mu$ defined over a metric space $(\mcX, c)$ with a unit diameter, parameters $p\ge 1, k>0$, and a constant $\alpha\in(0,1)$, let $\delta:=\|\mu- \tilde\mu\|_{TV}$. Then,
    \[\min\bigg\{\delta(1-\alpha), \frac{1}{k}W_{p,1-\delta(1-\alpha)}(\mu, \tilde{\mu})\bigg\}\le \ours{p,k}{\mu, \tilde{\mu}}\le \min\left\{\delta, \frac{1}{k}W_p(\mu, \tilde\mu)\right\}.\]
\end{restatable}
\begin{proof}
    Using Lemma~\ref{lemma:max_alpha_beta} on distributions $\mu$ and $\tilde{\mu}$,
    \begin{equation*}
        \ours{p,k}{\mu, \tilde{\mu}} \ge \min\bigg\{\delta(1-\alpha), \frac{1}{k}W_{p,1-\delta(1-\alpha)}(\mu, \tilde{\mu})\bigg\}.
    \end{equation*}
    Next, since $\delta=\|\mu- \tilde\mu\|_{TV}$, we get $W_{p,1-\delta}(\mu, \tilde\mu)=0$. Plugging into Lemma~\ref{lemma:max_alpha_beta},
    \begin{equation}\label{eq:sensitivity-1}
        \ours{p,k}{\mu, \tilde{\mu}} \le \max\{\delta, 0\}= \delta.
    \end{equation}
    Furthermore, since $W_{p,1-0}(\mu,\tilde\mu)=W_p(\mu,\tilde\mu)$, by Lemma~\ref{lemma:max_alpha_beta},
    \begin{equation}\label{eq:sensitivity-2}
        \ours{p,k}{\mu, \tilde{\mu}} \le \max\{0,  \frac{1}{k}W_p(\mu, \tilde\mu)\} = \frac{1}{k}W_p(\mu, \tilde\mu).
    \end{equation}
    Combining Equations~\eqref{eq:sensitivity-1} and~\eqref{eq:sensitivity-2},
    \begin{equation*}
        \ours{p,k}{\mu, \tilde{\mu}} \le \min\left\{\delta, \frac{1}{k}W_p(\mu, \tilde\mu)\right\},
    \end{equation*}
    as claimed.
\end{proof}

Assuming that the assumption (A1) holds for the distributions $\mu$ and $\tilde\mu$, by plugging $\alpha=0.9$ in Lemma~\ref{lemma:robustness-app}, we can derive the following lemma.

\robustnesss*

\subsubsection{Robustness to Sampling Discrepancies}\label{sec:convergence_appendix}

\begin{restatable}{lemma}{convergenceexcessd}\label{lemma:convergence_excess_d}
    For any distribution $\mu$ inside the unit $d$-dimensional hypercube, an empirical distribution $\mu_n$ of $\mu$, and a grid $\mcG$ with cell side length $n^{-\alpha}$, $\excess_\mu(\mcG)= \tilde{O}(n^{\frac{d\alpha}{2}-\frac{1}{2}})$ with a high probability. 
\end{restatable}
\begin{proof}
    First, note that if $\alpha\ge\frac{1}{d}$, then $n^{\frac{d\alpha}{2}-\frac{1}{2}} \ge 1$, and the lemma statement holds trivially. Therefore, from now on, we assume $\alpha$ to be less than $\frac{1}{d}$.
    For any cell $\cell$ of the grid $\mcG$, define $p_\cell:=\mu(\cell)$ to be the total probability mass of $\mu$ inside $\cell$, i.e., the probability that a point drawn from $\mu$ lies inside $\cell$. Any cell $\cell\in\mcG$ is considered a sparse cell if $p_\cell\le \frac{9\log n}{n}$, and a dense cell otherwise. Let $\mcG_\mcS$ (resp. $\mcG_\mcD$) denote the subset of sparse (resp. dense) cells of $\mcG$. For each sparse cell $\cell$,  $\excess_\mu(\cell)\le p_\cell\le \frac{9\log n}{n}$; therefore, using $\alpha<\frac{1}{d}$, the total contribution of sparse cells to the excess of $\mcG$ is at most \[O(|\mcG_\mcS|\times \frac{\log n}{n}) = O(n^{d\alpha}\times \frac{\log n}{n}) = \tilde{O}(n^{d\alpha-1})=\tilde{O}(n^{\frac{d\alpha}{2}-\frac{1}{2}}).\] 
    
    Next, we analyze the excess of the dense cells. Let $X=(x_1, \ldots, x_n)$ denote the set of $n$ samples drawn from $\mu$ that were used to construct the empirical distribution $\mu_n$. For each dense cell $\cell$, let $Y_\cell$ be a random variable denoting the number of samples in $X$ that lie inside $\cell$. 
    Using the Chernoff bound,
    \[Pr\big[Y_\cell\le np_\cell -3\sqrt{np_\cell\log n}\big]\le n^{-\frac{9}{2}}.\]
    In other words, for each $\cell\in\mcG_\mcD$, $\excess_\mu(\cell)=O(\frac{1}{n}\sqrt{np_\cell\log n})$ with probability at least $1-n^{-\frac{9}{2}}$.
    Therefore, with probability at least $(1-n^{-9/2})^{|\mcG_\mcD|}\ge (1-n^{-9/2})^{n}\ge 1-n^{-\frac{7}{2}}$, the total excess of the dense cells would be
    \begin{align*}
        \sum_{\cell\in\mcG_\mcD} \excess_\mu(\cell) &=O\left(\sum_{\cell\in\mcG_\mcD} \sqrt{\frac{p_\cell\log n}{n}}\right) = O\left(\sqrt{\frac{\log n}{n}}\sum_{\cell\in\mcG_\mcD} \sqrt{p_\cell}\right) \\ &= O\left(\sqrt{\frac{\log n}{n}}\times \sqrt{|\mcG_\mcD|}\right) = \tilde{O}(n^{\frac{d\alpha}{2}-\frac{1}{2}}),
    \end{align*}
    where the third equality holds since $\frac{\sum_{\cell\in\mcG}\sqrt{p_\cell}}{|\mcG_\mcD|}\le \sqrt{\frac{\sum_{\cell\in\mcG_\mcD}p_\cell}{|\mcG_\mcD|}}$ \cite{Sedrakyan2018} and $\sum_{\cell\in\mcG_\mcD}p_\cell\le 1$, and the last equality holds since $|\mcG_\mcD|\le n^{d\alpha}$.
\end{proof}

We obtain the following lemma by simply plugging $d=2$ in Lemma~\ref{lemma:convergence_excess_d}.

\convergenceexcess*

{\bf Improved proof of Lemma~\ref{lemma:convergence_2}.} We improve our bounds for the convergence rate of the empirical $(2,1)$-RPW by extending our approach and considering $O(\log\log n)$ grids instead of two grids. In the following, we construct a transport plan $\gamma$ that transports all except $\tilde{O}(n^{-\frac{1}{3}})$ mass with a cost of $\tilde{O}(n^{-\frac{1}{3}})$. We then conclude that $\ours{p,k}{\mu, \nu}=\tilde{O}(n^{-\frac{1}{3}})$. 

Without loss of generality, assume $n=2^{2^h}$ for some integer $h>0$. Let $\beta:=\frac{\log n}{3\log n - 2}$. Define a set of $h$ grids $\langle\mcG_1, \ldots, \mcG_h\rangle$, where each grid $\mcG_i$ has a side length $O(n^{-\alpha_i})$ for
$\alpha_i:= \frac{1}{2} - \beta\left(1-\frac{1}{2^i}\right)$.
We construct the grids in a way that their boundaries are aligned with each other. 

Let $\mu^0:=\mu$ and $\mu_n^0:=\mu_n$. Starting from $i=1$, we compute a partial transport plan $\gamma_i$ from $\mu_n^{i-1}$ to $\mu^{i-1}$ that transports as much mass as possible inside each cell of $\mcG_i$. We then define $\mu^i$ (resp. $\mu_n^i$) as the distribution of the mass of $\mu^{i-1}$ (resp. $\mu_n^{i-1}$) that is not transported by $\gamma_i$, set $i\leftarrow i+1$, and continue the same process until we process the last grid $\mcG_h$. Define $\gamma:=\sum_{i=1}^h \gamma_i$.
By our construction, the transport plan $\gamma$ transports as much mass as possible inside each cell of $\mcG_h$. Therefore, the total mass that is not transported by $\gamma$ is equal to the excess $\excess(\mcG_h)$, which by Lemma~\ref{lemma:convergence_excess}, with a high probability, is 
\begin{align}
    \tilde{O}(n^{\alpha_h-\frac{1}{2}}) &= \tilde{O}(n^{-\beta(1-\frac{1}{2^h})})=\tilde{O}(n^{\frac{-\log n+1}{3\log n - 2}})=\tilde{O}(n^{-\frac{1}{3} + \frac{1}{3(3\log n-2)}}) = \tilde{O}(n^{-\frac{1}{3}}).\label{eq:convergence_proof_2}
\end{align}
Next, we analyze the cost of $\gamma$ by computing the cost of each transport plan $\gamma_i$ separately. For $\gamma_1$, since each mass transportation is between points inside the same cell of $\mcG_1$ and has a squared cost of $O((n^{-\alpha_1})^2)$,
\begin{align*}
    w_2^2(\gamma_1) &= O(n^{-2\alpha_1}) = O(n^{-1 + \beta}) = O(n^{-\frac{2\log n - 2}{3\log n - 2}}) = O(n^{-\frac{2}{3} + \frac{2}{3(3\log n-2)}})=O(n^{-\frac{2}{3}}). 
\end{align*}
Furthermore, for each $i>1$, the transport plan $\gamma_i$ transports a total mass of at most $\mass{\mu^{i-1}}$, which is equal to the excess of the grid $\mcG_{i-1}$, and by Lemma~\ref{lemma:convergence_excess} is $\tilde{O}(n^{\alpha_{i-1}-\frac{1}{2}})$. Since $\gamma_i$ transports mass between points inside the same cell of $\mcG_i$, each mass transportation in $\gamma_i$ has a squared cost of $O(n^{-2\alpha_i})$, and therefore,
\begin{align*}
    w_2^2(\gamma_i) &= \tilde{O}(n^{\alpha_{i-1}-\frac{1}{2}-2\alpha_i}) = \tilde{O}(n^{-1+ \beta})=\tilde{O}(n^{-\frac{2}{3}}).
\end{align*}
Therefore, 
\begin{equation}
    w_2(\gamma) = \sqrt{\sum_{i=1}^h w_2^2(\gamma_i)} = \tilde{O}(\sqrt{hn^{-\frac{2}{3}}})= \tilde{O}(n^{-\frac{1}{3}}).\label{eq:convergence_proof_3}
\end{equation}
By Equations~\eqref{eq:convergence_proof_2} and~\eqref{eq:convergence_proof_3}, we have computed a transport plan $\gamma$ from $\mu_n$ to $\mu$ that, with a high probability, transports all except $\tilde{O}(n^{-\frac{1}{3}})$ mass with a cost $\tilde{O}(n^{-\frac{1}{3}})$. Therefore, $\ours{2}{\mu, \mu_n}=\tilde{O}(n^{-\frac{1}{3}})$ with a high probability.

We extend the same approach to any dimension $d\ge 2$ and any $p\ge 1$ in Lemma~\ref{lemma:convergence}.

\convergenceGeneral*
\begin{proof}
In Lemma~\ref{lemma:relationWasserstein}, we show that for any $p\ge 1$ and $k>0$, $\ours{p,k}{\mu, \nu}\le \frac{1}{k}W_p(\mu, \nu)$. Therefore, for any constant $k>0$, the convergence rate of the empirical $(p,k)$-RPW is upper-bounded by the convergence rate of the empirical $p$-Wasserstein distance. \citet{fournier2015rate} showed that when $p\le \frac{d}{2}$, the $p$-Wasserstein distance achieves a convergence rate of $\tilde{O}(n^{-\frac{1}{d}})$. Therefore, our metric also achieves a convergence rate of $\tilde{O}(n^{-\frac{1}{d}})$ in this case, proving the bound claimed in the lemma statement for $p\le \frac{d}{2}$. In the remaining of this proof, we prove our bounds for $p>\frac{d}{2}$.

Let $h=\log_{\frac{2p}{d}} \log_2 n$. We assume that $h$ is an integer. Let $\beta:=\frac{\log n}{(p+\frac{2p}{d}-1)\log n - p}$. Define a set of $h$ grids $\langle \mcG_1, \ldots, \mcG_h\rangle$, where each grid $\mcG_i$ has a cell side length of $n^{-\alpha_i}$ for \[\alpha_i:=\frac{1}{d}-\frac{2p}{d^2}\beta\left(1-(\frac{d}{2p})^i\right).\] 
Following the approach discussed in Section~\ref{sec:convergence}, we construct $h$ transport plans $\gamma_1,\ldots, \gamma_h$ and define a transport plan $\gamma:=\sum_{i=1}^h\gamma_i$ to be a transport plan that transports total mass of $\min\{\mu(\cell), \mu_n(\cell)\}$ inside each cell $\cell\in\mcG_h$. 
By Lemma~\ref{lemma:convergence_excess_d}, the total free mass with respect to $\gamma$ would be
\begin{align}
    1-\mcM(\gamma)=\excess(\mcG_h) &= \tilde{O}\left(n^{\frac{d\alpha_h}{2}-\frac{1}{2}}\right) = \tilde{O}\left(n^{\frac{1}{2}-\frac{p}{d}\beta + \frac{p}{d}\beta\cdot (\frac{d}{2p})^h - \frac{1}{2}}\right) = \tilde{O}\left(n^{-\frac{\log_2 n -1}{(d+2-\frac{d}{p})\log n - d}}\right)\nonumber\\ &= \tilde{O}\left(n^{-\frac{1}{d+2-\frac{d}{p}} +\frac{1-\frac{d}{d+2-\frac{d}{p}}}{(d+2-\frac{d}{p})\log n - d}}\right)= \tilde{O}\left(n^{-\frac{p}{(d+2)p-d}}\right).\label{eq:convergence-equal-1}
\end{align}
Next, we bound the cost of $\gamma$ by analyzing the cost of each transport plan $\gamma_i, i\in[1,h]$ separately.
For $\gamma_1$, each mass transportation is inside a cell of $\mcG_1$ and has a $p$th power cost of $\tilde{O}(n^{-p\alpha_1})$; hence,
\begin{equation}
    w_p^p(\gamma_1)=\tilde{O}(n^{-p\alpha_1})=\tilde{O}\left(n^{-\frac{p}{d}+\frac{2p^2}{d^2}\beta\left(1-\frac{d}{2p}\right)}\right)=\tilde{O}\left(n^{-\frac{p^2}{(d+2)p-d}}\right).
\end{equation}
Finally, for each $i>1$, the total mass transported by $\gamma_i$ is equal to the excess of $\mcG_{i-1}$, which by Lemma~\ref{lemma:convergence_excess} is $\tilde{O}(n^{\frac{d\alpha_{i-1}}{2}-\frac{1}{2}})$. Each mass transportation in $\gamma_i$ is between points within the same cell of $\mcG_{i}$ and thus has a $p$th power cost of $\tilde{O}(n^{-p\alpha_i})$. Therefore,
\begin{align}
    w_p^p(\gamma_i)&=\tilde{O}(n^{\frac{d\alpha_{i-1}}{2}-\frac{1}{2}-p\alpha_i})=\tilde{O}\left(n^{-\frac{p^2}{(d+2)p-d}}\right).
\end{align}
Adding the cost of all transport plans,
\begin{align}
    w_p(\gamma)&=\bigg(\sum_{i=1}^h w_p^p(\gamma_i)\bigg)^{1/p} = \tilde{O}\big((n^{-\frac{p^2}{(d+2)p-d}}\log\log n)^{1/p}\big) = \tilde{O}(n^{-\frac{p}{(d+2)p-d}}).\label{eq:convergence-equal-2}
\end{align}
Combining Equations~\eqref{eq:convergence-equal-1} and~\eqref{eq:convergence-equal-2}, $\ours{p}{\mu, \mu_n} = \tilde{O}(n^{-\frac{p}{(d+2)p-d}})$.
\end{proof}

\subsection{Missing Proofs of Section~\ref{sec:relation}.}\label{sec:relation-appendix}

To prove Lemma~\ref{lemma:relationLevy}, we begin by showing in Lemma~\ref{lemma:relationLevy-a} that when $\mu$ and $\nu$ are discrete distributions, $\ours{\infty,1}{\mu,\nu}=\lp{\mu,\nu}$. We then use Lemma~\ref{lemma:relationLevy-a} to show that the same also holds for continuous distributions.

\begin{restatable}{lemma}{relationLevy}\label{lemma:relationLevy-a}
    For any pair of discrete probability distributions $\mu$ and $\nu$ in a metric space $(\mcX, c)$ with a unit diameter, $\ours{\infty,1}{\mu, \nu}= \lp{\mu, \nu}$.
\end{restatable}
\begin{proof}
    To prove this lemma, we first show that $\ours{\infty, 1}{\mu, \nu}\le \lp{\mu, \nu}$. We then show that $\lp{\mu, \nu}\le \ours{\infty, 1}{\mu, \nu}$ and conclude the lemma statement. 
    
    Let $\delta:=\lp{\mu, \nu}$ and suppose $A$ and $B$ denote the support of $\mu'$ and $\nu'$, respectively. Define the $\delta$-disc graph $G_\delta$ between points in $A$ and $B$ to be a bipartite graph where for each pair $(a,b)\in A\times B$ with $c(a,b)\le \delta$, there is an edge between $a$ and $b$ in $G_\delta$. For any set $S$ of vertices of $G_\delta$, let $\mcN(S)$ denote the set of neighbors of $S$ in $G_\delta$. By the definition of the \levy distance, for any set of points $S\subseteq A$, $\mu(S)\le \nu(\mcN(S)) + \delta$. Similarly, for any subset $T\subseteq B$, $\nu(T)\le \mu(\mcN(T)) + \delta$. 

    We prove that $\ours{\infty, 1}{\mu, \nu}\le \lp{\mu, \nu}$ by showing that the maximum transport plan $\gamma$ on the $\delta$-disc graph $G_\delta$ transports a total mass of at least $1-\delta$. In this case, since all edges of $\gamma$ has a cost of at most $\delta$, $w_\infty(\gamma)\le \delta$. Hence, the $(1-\delta)$-partial $\infty$-Wasserstein distance from $\mu$ to $\nu$ would be at most $\delta$ and $\ours{\infty,1}{\mu, \nu}\le \delta$, as claimed.

    Consider a bipartite graph $G'_\delta$ obtained from $G_\delta$ by adding a fake vertex $b'$, where $b'$ has a mass of $\delta$ and is connected to all points of $A$ with a cost $\delta$. For any subset $S\subseteq A$ (resp. $T\subseteq B$), let $\mcN'(S)$ (resp. $\mcN'(T)$) denote the set of neighbors of $S$ (resp. $T$) in $G'_\delta$ and suppose $\mu'(S)$ (resp. $\nu'(T)$) denotes the total mass of points in $S$ (resp. $T$) for any subset $S\subseteq A$ (resp. $T\subseteq B\cup \{b'\}$). By construction, for any subset $S\subseteq A$, $\mu'(S)\le \nu'(\mcN'(S))$ and similarly, for any subset $T\subseteq B$, $\nu'(T)\le \mu'(\mcN'(T))$; hence, by the extension of Hall's marriage theorem \cite{bansil2021w}, there exists a transport plan $\gamma'$ on the graph $G'_\delta$ that transports all mass of the points in $A$ to the points in $B\cup\{b'\}$. Let $\gamma$ denote the transport plan obtained from $\gamma'$ after removing the fake vertex $b'$ and all mass transportation to $b'$. The transport plan $\gamma$ transports at least $1-\delta$ mass and has a cost $w_\infty(\gamma')\le \delta$.
    Therefore, if $\gamma$ transports a total mass of $1-\delta'$ for some $\delta'\le \delta$,
    \[W_{\infty, 1-\delta}(\mu, \nu)\le W_{\infty, 1-\delta'}(\mu, \nu) \le w_\infty(\gamma)\le \delta, \]
    which means that $\ours{\infty,1}{\mu, \nu}\le \delta=\lp{\mu, \nu}$.
    We next show that $\ours{\infty,1}{\mu, \nu}\ge \lp{\mu, \nu}$ in a similar way and conclude that $\ours{\infty,1}{\mu, \nu}= \lp{\mu, \nu}$.

    Let $\delta:=\ours{\infty,1}{\mu, \nu}$. Let $\gamma$ be a $(1-\delta)$-partial OT plan from $\mu$ to $\nu$ and let $G_\delta$ be a $\delta$-disk graph on $A\cup B$. Since $w_\infty(\gamma)\le \delta$, all mass transportation by $\gamma$ has a cost at most $\delta$, i.e., all edges carrying a positive mass in $\gamma$ are present in $G_\delta$. For any subset $S\subseteq A$, let $\mu_S$ be the distribution of the mass of $\mu$ on the points in $S$. Define $\nu_S := \gamma\#\mu_S$ to be the subset of mass of $\nu$ that is transported from $\mu_S$ according to $\gamma$, and let $T_S\subseteq B$ be the support of $\nu_S$. Recall that all edges carrying a positive mass in $\gamma$ are present in $G_\delta$; therefore, all points in $T_S$ are neighbors of $S$ in $G_\delta$, i.e., $T\subseteq \mcN(S)$ and $\nu(T_S)\le \nu(\mcN(S))$. Furthermore, since $\gamma$ is a $(1-\delta)$-partial OT plan, the total mass of $\mu_S$ that is not transported by $\gamma$ is at most $\delta$, and hence, 
    \[\mu(S)\le \nu(T_S)+\delta\le \nu(\mcN(S)) + \delta.\]
    One can also show that for each subset $T\subseteq B$, $\nu(T)\le \mu(\mcN(T))+\delta$ using an identical argument. Therefore, by the definition of the \levy distance, $\lp{\mu, \nu}\le\delta=\ours{\infty,1}{\mu,\nu}$.
\end{proof}

In the following, we show that for any pair of (continuous) probability distributions $\mu$ and $\nu$ and any $\varepsilon>0$,
\begin{equation}\label{eq:ours-lp-relation0}
    |\ours{\infty,1}{\mu, \nu}-\lp{\mu, \nu}|\le \varepsilon
\end{equation}
and conclude that $\ours{\infty,1}{\mu, \nu}=\lp{\mu, \nu}$ for any pair of probability distributions (discrete or continuous).

Define $\cell$ to be a unit $d$-dimensional hypercube containing the set $\mcX$. Let $\mcG$ be a grid of cell side length $\frac{\varepsilon}{4\sqrt{d}}$ that partitions $\cell$ into smaller cells. Using the grid $\mcG$, we construct two discrete distributions $\mu^\varepsilon$ and $\nu^\varepsilon$ as follows. Let $\mcG_\mcX$ denote the subset of cells of $\mcG$ that intersects the set $\mcX$. For each cell $\xi\in\mcG_\mcX$, we pick an arbitrary point $r_\xi$ inside $\xi\cap\mcX$ as the representative point of $\xi$. Let $\mcR:=\bigcup_{\xi\in\mcG}\{r_\xi\}$. Define $\mu^\varepsilon$ (resp. $\nu^\varepsilon$) as a discrete distribution over $\mcR$ that assigns, for each $\xi\in\mcG_\mcX$, a mass of $\mu(\xi)$ (resp. $\nu(\xi)$) to its representative point $r_\xi$. This completes the construction of $\mu^{\varepsilon}$ and $\nu^\varepsilon$.
Note that by Lemma~\ref{lemma:relationLevy}, 
\begin{equation}\label{eq:ours-lp-relation1}
    \ours{\infty,1}{\mu^\varepsilon, \nu^\varepsilon} = \lp{\mu^\varepsilon, \nu^\varepsilon}.
\end{equation} 
Furthermore, $W_\infty(\mu, \mu^\varepsilon)\le \frac{\varepsilon}{4}$, since there is a transport plan that transports the mass of $\mu$ inside each cell $\xi\in\mcG$ to the mass of $\mu^\varepsilon$ at $r_\xi$ and each mass transportation has a cost at most $\frac{\varepsilon}{4}$. From Lemma~\ref{lemma:max_alpha_beta}, 
$\ours{\infty, 1}{\mu, \mu^\varepsilon}\le \frac{\varepsilon}{4}$.
Similarly, $\ours{\infty, 1}{\nu, \nu^\varepsilon}\le \frac{\varepsilon}{4}$. Therefore, using the triangle inequality, 
\begin{equation}\label{eq:ours-lp-relation2}
    |\ours{\infty,1}{\mu,\nu}-\ours{\infty,1}{\mu^\varepsilon, \nu^\varepsilon}|\le \ours{\infty,1}{\mu,\mu^\varepsilon}+\ours{\infty,1}{\nu, \nu^\varepsilon}\le  \frac{\varepsilon}{2}.
\end{equation} 
One can also show in a similar way that \begin{equation}\label{eq:ours-lp-relation3}
    |\lp{\mu,\nu}-\lp{\mu^\varepsilon, \nu^\varepsilon}|\le \frac{\varepsilon}{2}.
\end{equation}
We conclude Equation~\eqref{eq:ours-lp-relation0} by combining Equations~\eqref{eq:ours-lp-relation1},~\eqref{eq:ours-lp-relation2}, and~\eqref{eq:ours-lp-relation3}.

\relationTV*
\begin{proof}
    We prove this lemma by first showing that $\ours{p,0}{\mu, \nu}\le \|\mu-\nu\|_{\mathrm{TV}}$ and then showing that $\|\mu-\nu\|_{\mathrm{TV}}\le \ours{p,0}{\mu, \nu}$.
    
    Let $\mcP(\mcX)$ denote the set of all probability distributions defined over the compact set $\mcX$. \citet{nietert2023outlier} showed that one can rewrite the $(1-\varepsilon)$-partial $p$-Wasserstein distance between $\mu$ and $\nu$ as
    \begin{equation}\label{eq:TV-relation-1}
        W_{p,1-\varepsilon}(\mu, \nu)=\inf_{\mu'\in\mcP(\mcX): \|\mu-\mu'\|_{\mathrm{TV}}\le \varepsilon}W_p(\mu', \nu).
    \end{equation}
    Define $\delta=\|\mu-\nu\|_{\mathrm{TV}}$. Plugging $\varepsilon=\delta$ in Equation~\eqref{eq:TV-relation-1},
    \begin{equation}\label{eq:TV-relation-2}
        W_{p,1-\delta}(\mu, \nu)=0.
    \end{equation}
    Therefore, by Lemma~\ref{lemma:max_alpha_beta},
    \begin{equation}\label{eq:TV-relation-5}
    \ours{p,0}{\mu, \nu}\le \max\{0,\delta\}=\delta=\|\mu-\nu\|_{\mathrm{TV}}.
    \end{equation}
    Next, let $\delta'=\ours{p,0}{\mu, \nu}$. By definition of the $(p,0)$-RPW,
    $W_{p,1-\delta'}(\mu, \nu)\le 0\times \delta' = 0$ (since the parameter $k$ is set to $0$), and since the partial $p$-Wasserstein distance is non-negative, $W_{p,1-\delta'}(\mu, \nu) = 0$. Therefore,
    \begin{equation}\label{eq:TV-relation-3}
        0=W_{p,1-\delta'}(\mu, \nu)=\inf_{\mu'\in\mcP(\mcX): \|\mu-\mu'\|_{\mathrm{TV}}\le \delta'}W_p(\mu', \nu).
    \end{equation}
    Let $\mu^*$ be the distribution realizing the infimum in Equation~\eqref{eq:TV-relation-3}. Then, $W_p(\mu^*, \nu)=0$, and by the metric properties of the $p$-Wasserstein distance, $\mu^*=\nu$; hence, 
    \begin{equation}\label{eq:TV-relation-4}
        \|\mu-\nu\|_{\mathrm{TV}}=\|\mu-\mu^*\|_{\mathrm{TV}}\le \delta'=\ours{p,0}{\mu, \nu}.
    \end{equation}
    Combining Equations~\eqref{eq:TV-relation-5} and~\eqref{eq:TV-relation-4}, $\ours{p,0}{\mu, \nu} = \|\mu-\nu\|_{\mathrm{TV}}$.
\end{proof}

\relationWasserstein*
\begin{proof}
    Let $\delta':=W_p(\mu, \nu)$. In this case, 
    \[W_{p, 1-\min\{1, \frac{\delta'}{k}\}}(\mu, \nu)\le W_{p}(\mu, \nu)=k\times\frac{\delta'}{k}.\]
    Therefore, by Lemma~\ref{lemma:max_alpha_beta}, $\ours{p,k}{\mu, \nu}\le \max\{\min\{1, \frac{\delta'}{k}\}, \frac{\delta'}{k}\}=\frac{1}{k}W_p(\mu, \nu)$. 
    
    We next show that $\frac{1}{k}W_p(\mu, \nu)\le \ours{p,k}{\mu, \nu}+k^{-\frac{p+1}{p}}$. Note that the inequality holds trivially for any $k\le 1$, since $k^{-\frac{p+1}{p}}\ge \frac{1}{k}\ge\frac{1}{k}W_p(\mu, \nu)$. We therefore assume that $k> 1$. Let $\delta=\ours{p,k}{\mu, \nu}$.
    Since the $(1-\frac{1}{k})$-partial $p$-Wasserstein distance is at most $1$, by Lemma~\ref{lemma:max_alpha_beta}, $\delta \le \max\{\frac{1}{k}, \frac{1}{k}W_{p, 1-\frac{1}{k}}(\mu, \nu)\}\le\frac{1}{k}$. Let $\gamma$ be a $(1-\delta)$-partial OT plan.
    Since the underlying metric space has a unit diameter, the remaining $\delta$ mass of $\mu$ and $\nu$ with respect to $\gamma$ can be transported at a cost at most $\delta$; therefore,
    \begin{align*}
        W_p(\mu, \nu)\le \big(w_{p}^p(\gamma) + \delta\big)^{1/p}\le \big((k\delta)^p + \frac{1}{k}\big)^{1/p}\le k\delta + k^{-1/p}.
    \end{align*}
    Equivalently, $\frac{1}{k}W_p(\mu, \nu)\le \ours{p,k}{\mu, \nu} + k^{-\frac{p+1}{p}}$.
\end{proof}

\subsection{Missing Details of Section~\ref{sec:approx}.}\label{sec:approx-ap}
In this section, we provide the details of the algorithms mentioned in Section~\ref{sec:approx}.

{\bf Highly-Accurate Algorithm.} In this algorithm, we obtain an approximation of our metric by a simple guessing procedure as follows. Starting from an initial guess $g_1=0.5$ for the value of our metric, at any step $i$ of our algorithm and for any guess value $g_i\ge 0$, define $w_i:=W_{p,1-g_i}(\mu, \nu)$. If $w_i\le kg_i$, then by Lemma~\ref{lemma:max_alpha_beta}, $w_i\le \ours{p,k}{\mu, \nu}\le g_i$, i.e., our guess value is large and we set $g_{i+1} \leftarrow g_i - 2^{-(i+1)}$.
Otherwise, $w_i>kg_i$, and in this case, by Lemma~\ref{lemma:max_alpha_beta}, $g_i\le \ours{p,k}{\mu, \nu}< w_i$, i.e., our guess value is small and we set $g_{i+1} \leftarrow g_i + 2^{-(i+1)}$. Note that at any step $i$, $|g_i-\ours{p,k}{\mu, \nu}|\le 2^{-i}$. Therefore, to obtain a $\delta$-additive approximation of the $(p,k)$-RPW, the algorithm returns the guess value $g_i$ when $2^{-i}\le \delta$. This completes the description of our algorithm.

We next analyze the running time of this algorithm. Computing the $(1-g_i)$-partial $p$-Wasserstein distance can be done using a standard OT solver in an augmented space~\cite{chapel2020partial}, which takes $O(n^3\log n)$ time~\cite{edmonds1972theoretical, orlin1988faster}. The total number of iterations of our algorithm is $O(\log \delta^{-1})$ and therefore, our algorithm runs in $O(n^3\log n\log\delta^{-1})$ time.

{\bf Computing Through an Approximate OT-Profile.} In this part, we show that an approximation of the OT-profile can be used to approximate our metric. We then conclude that one can use the LMR algorithm to obtain such approximations of the OT-profile and to obtain a $\delta$-additive approximation of the RPW distance in $O(\frac{n^{2}}{\delta^p} + \frac{n}{\delta^{2p}})$ time. 

For a $\delta'\in(0,1]$, $p\ge1$, and $\alpha\in[0,1]$, let $\overline{W}_{p,\alpha}(\mu, \nu)$ denote a $\delta'$-close $\alpha$-partial $p$-Wasserstein distance, i.e., $W_{p,\alpha}(\mu, \nu)\le \overline{W}_{p,\alpha}(\mu, \nu)\le W_{p,\alpha}(\mu, \nu) + \delta'$. Define
\[\oursapprox{p,k}{\mu, \nu} = \min\{\varepsilon\ge 0 \mid\overline{W}_{p,1-\varepsilon}(\mu, \nu)\leq k\varepsilon\}\]
to be the $(p,k)$-RPW distance function when computed using the approximate partial $p$-Wasserstein distances. In the following lemma, we show that $\oursapprox{p,k}{\mu, \nu}$ is a $\frac{2\delta'}{k}$-additive approximation of $\ours{p,k}{\mu, \nu}$.

\begin{restatable}{lemma}{approx}\label{lemma:approx}
    For any pair of distributions $\mu$ and $\nu$ in a metric space $(\mcX, c)$ with a unit diameter and any parameters $p\ge1$, $k>0$, and $\delta'>0$, \[\ours{p,k}{\mu, \nu}\le \oursapprox{p,k}{\mu, \nu}\le \ours{p,k}{\mu, \nu} + \frac{2\delta'}{k}.\]
\end{restatable}
\begin{proof}
    Let $\overline\delta:=\oursapprox{p,k}{\mu, \nu}$. By definition,    \begin{equation}\label{eq:additive_error_0}
        W_{p,1-\overline\delta}(\mu, \nu) \le \overline{W}_{p,1-\overline\delta}(\mu, \nu) \le k\overline\delta.
    \end{equation}
    Therefore, $\ours{p,k}{\mu, \nu}\le \overline\delta= \oursapprox{p,k}{\mu, \nu}$.
    Next, let $\delta:=\ours{p,k}{\mu, \nu}$. By definition, \begin{equation}\label{eq:additive_error_1}
        \overline{W}_{p,1-\delta}(\mu, \nu) \le W_{p,1-\delta}(\mu, \nu) + \delta' \le k\delta + \delta'.
    \end{equation}
    By properties of the partial $p$-Wasserstein distance,
    \begin{align}
        \overline{W}_{p,1-\delta-\frac{2\delta'}{k}}(\mu, \nu)&\le W_{p,1-\delta-\frac{2\delta'}{k}}(\mu, \nu) + \delta'\le W_{p,1-\delta}(\mu, \nu)+ \delta'\le \overline{W}_{p,1-\delta}(\mu, \nu)+ \delta'.\label{eq:additive_error_2}
    \end{align}
    Combining Equations~\eqref{eq:additive_error_1} and~\eqref{eq:additive_error_2},
    \begin{equation*}
        \overline{W}_{p,1-\delta-\frac{2\delta'}{k}}(\mu, \nu) \le \overline{W}_{p,1-\delta}(\mu, \nu)+\delta'\le k(\delta+\frac{2\delta'}{k}).
    \end{equation*}
    Therefore, $\oursapprox{p,k}{\mu,\nu} \le \delta + \frac{2\delta'}{k}=\ours{p,k}{\mu, \nu} + \frac{2\delta'}{k}$.
\end{proof}

\section{Additional Experiment Results of Section~\ref{sec:experimental}}\label{sec:appendix-exp}
In this section, we present the results of our experiments on the COREL dataset for the task of image retrieval.

\begin{figure}[htp]
\centering
\begin{tabular}{c@{\hskip 0pt}c@{\hskip 0pt}c}
     \includegraphics[width=.32\linewidth]{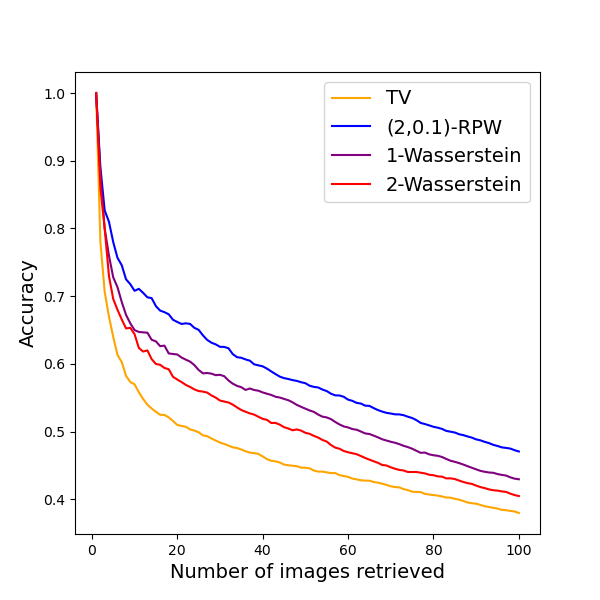} & \includegraphics[width=.32\linewidth]{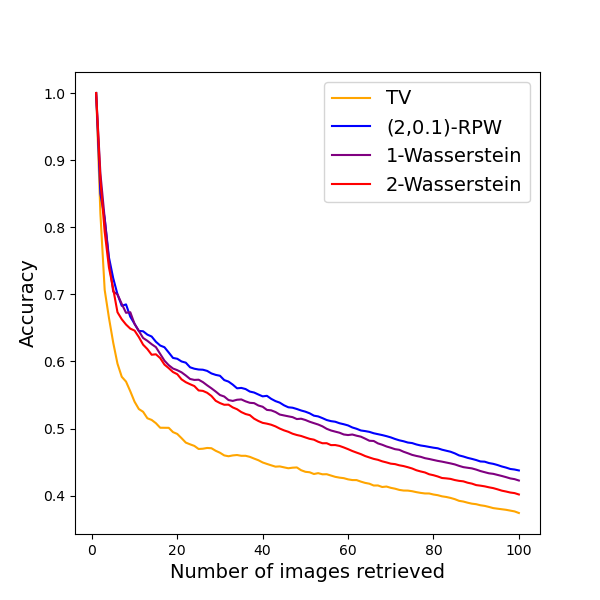} & \includegraphics[width=.32\linewidth]{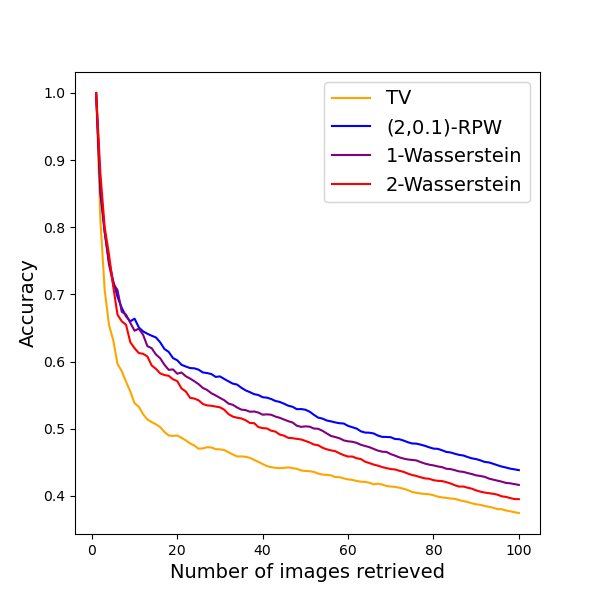}\\
     (i) Noise in datasets & (ii) Shift in datasets & (iii) Noise and shift in datasets
\end{tabular}
\vspace{0.2em}
\vspace{-1em}
\caption{The results of our experiments on image retrieval on the COREL dataset.}
\label{fig:converge_rate_app}
\end{figure}

Similar to the CIFAR-10 dataset, the COREL dataset also consists of color images, where images with the same labels may have significant variations and shifts. Due to these variations, $2$-Wasserstein and TV distances achieve lower accuracy in comparison to the $1$-Wasserstein distance. The $(2,0.1)$-RPW distance, however, outperforms the $1$-Wasserstein distance for this dataset.

\end{document}